\documentclass[pdflatex, sn-chicago, iicol]{sn-jnl}
\makeatletter\if@twocolumn\PassOptionsToPackage{switch}{lineno}\else\fi\makeatother

\jyear{2021}%

\usepackage{amsfonts}
\theoremstyle{thmstyleone}%
\newtheorem{theorem}{Theorem}
\newtheorem{proposition}[theorem]{Proposition}%
\newtheorem{lemma}[theorem]{Lemma}%
\theoremstyle{thmstylethree}%
\newtheorem{definition}{Definition}%
\newtheorem{assumption}{Assumption}%
\newtheorem{remark}{Remark}%

\begin{document}

\title[Risk-averse Contextual Multi-armed Bandit Problem with Linear Payoffs]{Risk-averse Contextual Multi-armed Bandit Problem with Linear Payoffs}


\author{\fnm{Yifan} \sur{Lin}}\email{ylin429@gatech.edu}
\equalcont{These authors contributed equally to this work.}

\author{\fnm{Yuhao} \sur{Wang}}\email{yuhaowang@gatech.edu}
\equalcont{These authors contributed equally to this work.}

\author*{\fnm{Enlu} \sur{Zhou}}\email{enlu.zhou@isye.gatech.edu}

\affil{\orgdiv{ISyE}, \orgname{Georgia Institute of Technology}, \orgaddress{\street{755 Ferst Dr NW}, \city{Atlanta}, \postcode{30332}, \state{GA}, \country{USA}}}




\abstract{In this paper we consider the contextual multi-armed bandit problem for linear payoffs under a risk-averse criterion. At each round, contexts are revealed for each arm, and the decision maker chooses one arm to pull and receives the corresponding reward. In particular, we consider mean-variance as the risk criterion, and the best arm is the one with the largest mean-variance reward. We apply the Thompson Sampling algorithm for the disjoint model, and provide a comprehensive regret analysis for a variant of the proposed algorithm. For $T$ rounds, $K$ actions, and $d$-dimensional feature vectors, we prove a regret bound of $O((1+\rho+\frac{1}{\rho}) d\ln T  \ln \frac{K}{\delta}\sqrt{d K  T^{1+2\epsilon} \ln \frac{K}{\delta} \frac{1}{\epsilon}})$ that holds with probability $1-\delta$ under the mean-variance criterion with risk tolerance $\rho$, for any $0<\epsilon<\frac{1}{2}$, $0<\delta<1$. The empirical performance of our proposed algorithms is demonstrated via a portfolio selection problem.}

\keywords{multi-armed bandit, context, risk-averse, Thompson sampling}



\maketitle
\section{Introduction}
\label{sec: intro}
The multi-armed bandit (MAB) problem is a classic online decision-making problem with limited feedback. In the standard MAB problem, in each of $T$ rounds, a decision maker plays one of the $K$ arms and receives a reward (also called payoff) of that arm. It has a wide variety of real-world applications, such as clinical trials, online advertisement, and portfolio selection. In certain situations, the decision maker may also be provided with contexts (also known as covariates or side information). For example, in personalized web services, the decision maker also knows the demographic, geographic, and behavioral information of the user (see \cite{li2010contextual}), which may be useful to infer the conditional average reward of an arm and allows the decision maker to personalize decisions for every situation and even improve the average reward over time. In this paper, we consider the contextual MAB problem. As opposed to the standard MAB problem, before making the choice of which arm to play, the decision maker observes a $d$-dimensional context $x_i$ associated with each arm $i$, and chooses an arm to play in the current round based on the rewards of the arms played in the past along with the contexts. In this paper, we assume the expected reward of each arm is linear in the context, i.e., we assume there is an underlying mean parameter $\mu_i \in \mathbb{R}^d$ for each arm $i$, such that the expected reward for each arm $i$ takes the form $x_i^{\top} \mu_i$. A class of predictors is said to be linear if each predictor predicts which arm gives the best expected reward that is linear in the observed context. The linear assumption leads to a succinct and tractable representation and is enough for good real-world performance (see \cite{li2010contextual}). The goal of the decision maker is to minimize the so-called regret, with respect to the best linear predictor in hindsight, who predicts exactly $\mu_i$ and pulls the arm with the largest expected reward $x_i^{\top} \mu_i, i \in [K]$.

The MAB (or contextual MAB) problem essentially seeks a trade-off between exploitation (of the current information by playing the arm with the highest estimated reward) and exploration (by playing other arms to collect reward information). The majority of the literature balance this trade-off by designing algorithms that maximize the expected total reward (or equivalently, minimize the expected total regret). However, in many real-world problems, maximizing the expected reward is not always the most desirable. For example, in the portfolio selection problem, some portfolio managers are risk-averse and prefer less risky portfolios with low expected return rather than highly risky portfolios with high expected return. In this case, the risk of the reward should also be taken into consideration. Motivated by such risk consideration in real-world problems, we take a risk-averse perspective on the stochastic contextual MAB problem. Although many risk measures have been used in the risk-averse MAB problems, we focus on the mean-variance paradigm (see \cite{markowitz1952portfolio}), given its advantages in interpretability, computation, and popularity among practitioners. To the best of our knowledge, we are among the first to consider the risk-averse contextual MAB problem. 

To solve risk-averse contextual MAB, we propose algorithms based on Thompson Sampling (TS, see \cite{thompson1933likelihood}). TS is one of the earliest heuristics for the MAB problems via a Bayesian perspective. Intuitively speaking, TS assumes a prior distribution on the underlying parameters of the reward distribution for each arm and updates the posterior distributions after pulling the arms. At each round, it samples from the posterior distribution for each arm, and plays the arm that produces the best sampled reward. Most of the existing literature that apply TS to the MAB problem does not care about the variance of the reward distribution, as they intend to maintain a low expected regret. However, in the risk-averse setting, the variance also plays a vital role in determining the best arm. Hence, in addition to sampling the reward mean, we also need to sample the reward variance. It poses great challenges to the Bayesian updating of the parameters as well as the regret analysis, as one also has to bound the deviation of the sampled reward variance. 
We then provide the theoretical analysis of a variant of the proposed algorithms and show a $O((1+\rho+\frac{1}{\rho}) d\ln T  \ln \frac{K}{\delta}\sqrt{d K  T^{1+2\epsilon} \ln \frac{K}{\delta} \frac{1}{\epsilon}})$ regret bound with high probability $1-\delta$, under some mild assumptions on the reward. The variant has the same posterior updating rule but carries out the sampling process differently, due to the technical difficulty in the regret analysis.

\subsection{Literature Review}\label{subsec: lit}
The contextual MAB problem is widely studied for online decision making.
Two principled approaches that balance the trade-off between exploitation and exploration in the contextual MAB problem are upper confidence bound (UCB) and TS. The UCB algorithm addresses the problem from a frequentist perspective, which captures the uncertainty of the reward distribution by the width of the confidence bound. Most notably, \cite{auer2002using} considers the contextual MAB problem with linear reward, and presents LinRel, a UCB-type algorithm that achieves an $\tilde{O}(\sqrt{Td})$ high probability regret bound. \cite{li2010contextual} present LinUCB, which utilizes ridge regression to estimate the mean parameters and is much easier to solve than LinRel. \cite{chu2011contextual} later show an $O\left(\sqrt{T d \ln ^{3}(K T \ln (T) / \delta)}\right)$ regret bound that holds with probability $1-\delta$ for a variant of LinUCB. \cite{abbasi2011improved} modify the UCB algorithm in \cite{auer2002using} and improve the regret bound by a logarithmic factor. 

On the other hand, TS is Bayesian approach to the same problem and is shown to be competitive to or better than the UCB algorithms (see \cite{chapelle2011empirical}). Most notably, \cite{agrawal2013thompson} provide the first theoretical guarantee for the TS algorithm on the contextual MAB problem, and shows an $\tilde{O}(d^{3/2}\sqrt{T})$ (or $\tilde{O}(d\sqrt{T \log(K)})$) high probability regret bound. Note that in the regret analysis of \cite{agrawal2013thompson}, the reward distribution is only assumed to be sub-Gaussian, while the Bayesian updating of the posterior parameters assumes a Gaussian likelihood with an unknown mean but known variance. We also assume a Gaussian likelihood, but with (unknown mean and) unknown variance. It is well known that the conjugate prior (see \cite{schlaifer1961applied}) for the Gaussian likelihood is normal-gamma. Later in the paper we will show that in the case when the reward distribution is not Gaussian, we can still use the normal-gamma posterior updating rule such that the normal-gamma posterior converges in distribution to a point mass at the mean and variance of the true reward distribution.

The aforementioned literature on the contextual MAB problem all assume the best arm is chosen according to its expected performance. In some applications, however, the risk of the reward should also be taken into consideration. The risk-aware (or risk-averse) MAB problems have drawn increasing attention recently. \cite{sani2012risk} consider the mean-variance risk criterion and present a UCB-type algorithm (MV-UCB). Later \cite{vakili2015mean}, \cite{vakili2016risk} further complete the regret analysis of the MV-UCB algorithm. It should be noted that their definition of the regret is with respect to the mean-variance over a given horizon, which does not generalize to the contextual setting because the mean of each arm also depends on the context given at each round. \cite{galichet2013exploration} consider the conditional value at risk (CVaR, see \cite{rockafellar2000optimization}) criterion and present the multi-armed risk-aware bandit (MaRaB), a UCB-type algorithm. However, while they choose to pull the arm with the largest empirical CVaR reward, they consider the expected regret (rather than CVaR) and analyze the regret under the assumption that $\alpha=0$ and that the CVaR and average criteria coincide. \cite{maillard2013robust} consider entropic risk measure and present RA-UCB algorithm that achieves logarithmic regret. \cite{zimin2014generalized} consider risk measures as a continuous function of reward mean and reward variance, and present a UCB-type algorithm that achieves logarithmic regret. We extend their regret definition to the contextual setting, while one should note that the optimal arm at each round may vary, depending on the given contexts. \cite{cassel2018general} provide a more systematic approach to analyzing general risk criteria within the stochastic MAB formulation, and design a UCB-type algorithm that reaches $O\left(\sqrt{T}\right)$ regret bound. \cite{liu2020risk} design a UCB-type algorithm for the mean-variance MAB problem with Gaussian rewards. \cite{khajonchotpanya2021revised} develop a UCB-based algorithm that maximizes the expected empirical CVaR. On the other hand, even though TS is empirically shown (see \cite{baudry2021optimal}) to outperform UCB-type algorithms that often suffer from non-optimal confidence bounds, it has been considered only very recently for the risk-averse setting, due to its lack of theoretical understanding and difficulty of analyzing the regret bound. \cite{chang2020risk} consider to minimize the expected cost with CVaR constraint on the cost (threshold constraint) using CVaR-TS algorithm, which achieves improvements over the compared UCB-based algorithms for Gaussian bandits with finite variances. \cite{zhu2020thompson} design TS algorithms for Gaussian bandits and Bernoulli bandits, with the optimization objective being the mean-variance. \cite{ang2021thompson} present a Thompson sampling algorithm to minimize the entropic risk for Gaussian MABs, using similar tricks in \cite{zhu2020thompson}. 

\subsection{Contributions and Outlines}
In this paper, we consider contextual MAB problems with linear rewards under the mean-variance risk criterion. Our contributions are summarized as follows:
\begin{itemize}
    \item We develop a TS algorithm, namely mean-variance TS disjoint (MVTS-D) for the disjoint model. 
    \item We present the theoretical analysis of a variant of the proposed algorithm under some mild assumptions. The regret analysis is a non-trivial extension of the existing results in \cite{agrawal2013thompson} to the risk-averse setting, as the considered mean-variance criterion greatly complicates the regret analysis.
    \item We carry out extensive sets of simulations on different reward distributions to demonstrate the performance of our proposed algorithms in a portfolio selection example. 
\end{itemize}

The rest of the paper is organized as follows. Section~\ref{sec: preliminary} introduces the contextual MAB problem under the mean-variance paradigm. Section~\ref{sec: algorithm} 
presents the Thompson Sampling algorithm for the disjoint model. We then conduct the regret analysis for a variant of the proposed algorithm in Section~\ref{sec: analysis}. Section~\ref{sec: numerical} demonstrates the empirical performance of the proposed algorithms with a portfolio selection example. Section~\ref{sec: conclusion} concludes the paper.

\section{Problem Setting}
\label{sec: preliminary}
Consider a contextual MAB problem with $K$ arms. At each round $t=1,2,\cdots,T$, a context $x_i(t) \in \mathbb{R}^{d}$ is revealed for each arm $i \in K$. The contexts can be chosen arbitrarily by an adversary. After observing the contexts, the decision maker plays one of the $K$ arms $a(t)$ and receives the reward $r_{a(t)}(t)$. We assume the reward for arm $i$ at round $t$ is generated from an unknown distribution $\nu_i (\mu_i,\sigma_i^2)$ with mean $x_i(t)^{\top}\mu_i$ and variance $\sigma^2_i$, where $\mu_i \in \mathbb{R}^{d}$ is a fixed but unknown mean parameter for the arm $i$, and $\sigma^2_i$ is a fixed but unknown variance parameter for the arm $i$. This model is called \emph{disjoint} since the mean parameter and variance parameter are not shared among different arms. 
Define the history $\mathcal{H}_{t}=\{x_{i}(t), a(\tau), r_{a(\tau)}(\tau), i=1,\cdots,K, \tau=1,\cdots,t\}$, which summarizes all the information of contexts, pulled arms, and corresponding rewards up to round $t$. All reward samples are independent conditioned on the choice of the arm and the context, and thus
\begin{align*}
    & ~~~~\mathbb{E}[r_{a(t)}(t)\mid\{x_i(t)\}_{i=1}^{K},a(t),\mathcal{H}_{t-1}]\\
    &=\mathbb{E}[r_{a(t)}(t)\mid a(t),x_{a(t)}(t)]= x_{a(t)}(t)^{\top}\mu_{a(t)},
\end{align*}
and similarly 
\begin{align*}
    & ~~~\operatorname{Var}[r_{a(t)}(t)\mid\{x_i(t)\}_{i=1}^{K},a(t),\mathcal{H}_{t-1}]\\
    &=\operatorname{Var}[r_{a(t)}(t)\mid a(t),x_{a(t)}(t)]= \sigma^2_{a(t)}.
\end{align*}

\begin{definition}
The mean-variance of arm $i$ at round $t$ with mean $x_i(t)^{\top} \mu_i$, variance $\sigma^2_i$, referred to as regret mean and regret variance, and risk tolerance $\rho$, is denoted by $\operatorname{MV}_i(t):=x_i(t)^{\top} \mu_i - \rho \sigma^2_i$. 
\end{definition}

It should be noted that different from the traditional contextual MAB problem, the criterion of choosing an optimal arm is based on the mean-variance performance. The risk tolerance parameter reflects the risk attitude of the decision maker. When $\rho=0$, the decision maker is risk-neutral, and our problem is reduced to the traditional contextual MAB with a risk-neutral criterion. When $\rho \to \infty$, the decision maker hates the risk so much that our problem turns to a variance minimization problem, which can also be easily handled by only sampling regret variance and choosing the arm with the smallest sampled regret variance to pull. Hence, we focus on the more interesting case $0<\rho<\infty$. As a final note, our algorithm and analysis can also be easily extended to the case $\rho<0$, which means the decision maker is risk-seeking. 

A policy, or allocation strategy $\pi$, is an algorithm that chooses at each round $t$, an arm $a(t)$ to pull, based on the history $\mathcal{H}_{t-1}$ and the context $x_i(t)$ for $i \in [K]$. Let $a^{*}(t)$ denote the optimal arm to pull at round $t$ under the mean-variance criterion, i.e., $a^{*}(t) := \arg \max_{i \in [K]} x_i(t)^{\top}\mu_i - \rho \sigma_i^2$. Let $\Delta_i(t)$ denote the difference between the mean-variances of the optimal arm $a^{*}(t)$ and arm $i$, i.e., 
\begin{align*}
    \Delta_i(t) = \operatorname{MV}_{a^{*}(t)}(t) - \operatorname{MV}_i(t). 
\end{align*}

The regret at round $t$ is defined as $R(t) = \Delta_{a(t)}(t)$, where $a(t)$ is the arm to pull at round $t$, determined by the algorithm $\pi$. The goal is to minimize the total regret $\mathcal{R}(T)=\sum_{t=1}^{T}R(t)$, or in other words, design an algorithm whose regret increases as slowly as possible as $T$ increases. Note that the optimal policy may not be a single-arm policy, since the optimal arm at each round $t$ is determined by the given contexts.

\section{Thompson Sampling}
\label{sec: algorithm}
In this section, we consider TS for the risk-averse contextual MAB problem under the disjoint model. In case the true reward distribution is Gaussian, it is natural to use Gaussian likelihood (for the rewards) and normal-gamma conjugate prior (for the mean and variance parameters) to design our Thompson Sampling algorithm. For a more general reward distribution that may not be Gaussian, usually we can take two approaches. First, we can choose the same likelihood as the true reward distribution, if the latter is known, and update the posterior distribution accordingly. The resulting posterior, however, is usually intractable, and one may choose to use Markov Chain Monte Carlo (MCMC) to sample from the posterior. The latest work along this line can be found in \cite{xu2022langevin}. An alternative approach, especially when the true reward distribution is unknown, is that we still choose the Gaussian likelihood, and then apply the variational Bayes (VB) technique to obtain a tractable approximate posterior. We show the details of VB under model mis-specification in the subsection below, where model mis-specification refers to that the likelihood (the chosen model) is different from the true reward distribution (the true model).

\subsection{Variational Bayes under Model Mis-specification}
For ease of exposition, we first introduce some notations and a simplified problem setting. Denote the mean and variance parameters of the true reward distribution by $\theta=(\mu,\sigma^2)$. The reward $y_1,\cdots,y_n$ are $n$ data points independently and identically distributed (i.i.d.) with a true density $p_0(\cdot)$. The Gaussian likelihood is denoted by $p(y \mid \theta)$. The mean field variational Bayes (MFVB) approximates the exact posterior distribution (updated using the likelihood) by a probability distribution with density $q(\theta)$ belonging to some tractable family of distributions $\mathcal{Q}$ that are factorizable, i.e., $\mathcal{Q}=\left\{q(\theta): q(\theta)= q_{1}\left(\mu\right)q_{2}\left(\sigma^2\right)\right\}$. The optimal VB posterior $q^{*}(\theta)$ is then found by minimizing the Kullback-Leibler (KL, see \cite{kullback1951information}) divergence from the exact posterior distribution $p(\theta \mid y_1,\cdots,y_n)$ to $\mathcal{Q}$, i.e., 
\begin{align*}
    q^{*}(\theta) =  \underset{q (\theta) \in \mathcal{Q}}{\arg \min }\Big\{& \operatorname{KL}(q \| p(\theta \mid y_1,\cdots,y_n))\\
    &:=\int q(\theta) \log \frac{q(\theta)}{p(\theta \mid y_1,\cdots,y_n)} d \theta\Big\}.
\end{align*}

It is shown in \cite{tran2021practical} that the optimal VB posterior takes the form of normal-gamma by choosing a normal-gamma prior and the Gaussian likelihood. It should be noted that the VB posterior is an approximation to the exact posterior. According to the Bernstein-Von Mises theorem under the model mis-specification, the exact posterior converges in distribution to a point mass at $\theta^{*}$ (\cite{kleijn2012bernstein}), where $\theta^{*}$ is the value of $\theta$ that minimizes the KL divergence between the assumed likelihood and the true reward distribution, i.e.,
\begin{align}
    \theta^{*}=\underset{\theta}{\arg \min } \operatorname{KL}\left(p_{0}(y) \| p(y \mid \theta)\right).
\label{eq: theta_star}
\end{align}

\cite{wang2019variational} later show that the VB posterior also converges in distribution to a point mass at $\theta^{*}$, and the VB posterior mean converges almost surely to $\theta^{*}$. It should be noted that since we choose the Gaussian likelihood, the parameter $\theta^{*}$ that minimizes the KL divergence in \eqref{eq: theta_star} is exactly the same as the parameter obtained by setting the first two moments of the Gaussian likelihood to the first two moments of the true reward distribution (see \cite{minka2001expectation, kurz2016kullback}), i.e., $\theta^{*}$ corresponds to the mean and variance parameters of the true reward distribution. Since we consider the mean-variance objective, we only care about the accuracy of the mean and variance estimates of the reward distribution. Hence, the normal-gamma posterior updating is justified by its convergence to the mean and variance of the true reward distribution, even under the model mis-specification. Hence, in the rest of the paper, we consider the Gaussian likelihood with a normal-gamma prior on the mean and variance parameters. 

\subsection{TS for the Disjoint Model}
Suppose the likelihood of reward $r_i(t)$ for arm $i$ at round $t$, given the context $x_i(t)$, the mean parameter $\mu_i$ and the precision parameter $\lambda_i$ (reciprocal of the variance parameter $\sigma_i^2$), were given by the probability density function (p.d.f.) of the Gaussian distribution $\mathcal{N}(x_{i}(t)^{\top}\mu_i, \lambda_i^{-1})$. Let $T_i(t)$ be the set of the rounds that arm $i$ has been pulled during the first $t-1$ rounds. We show the Bayesian updating of the parameters in the next proposition. 

\begin{proposition}\label{prop: posterior_disjoint}
Suppose the prior for $\lambda_i$ at round $t$ is given by Gamma$(C_i(t),D_i(t))$, and conditioned on $\lambda_i$, the prior for $\mu_i$ at round $t$ is given by $\mathcal{N}(A_i(t)^{-1}b_i(t), (\lambda_i A_i(t))^{-1})$. Here $C_i(t)$ is the shape parameter and $D_i(t)$ is the rate parameter of the Gamma distribution. Let
\begin{align*}
    A_{i}(t)= \mathbf{I}_d+\sum_{s \in T_i(t)}x_i(s)x_i(s)^{\top},
\end{align*} 
\begin{align*}
    b_i(t)= \mathbf{0}_{d \times 1} + \sum_{s \in T_i(t)}x_i(s)r_i(s),
\end{align*}
where $\mathbf{I}_d$ is a $d$-dimensional identity matrix, $\mathbf{0}_{d \times 1}$ is a $d$-dimensional zero vector. Then the posterior for $\lambda_i$ is given by Gamma($C_i(t+1),D_i(t+1)$), and conditioned on $\lambda_i$, the posterior for $\mu_i$ is given by $\mathcal{N}(A_i(t+1)^{-1}b_i(t+1), (\lambda_i A_i(t+1))^{-1})$, where
\begin{align*}
    C_i(t+1)=C_i(t) + \frac{1}{2},
\end{align*}
\begin{align*}
    D_i(t+1) = D_i(t) + \frac{1}{2} & [-b_i(t+1)^{\top}A_i(t+1)^{-1}b_i(t+1) \\
    & + b_i(t)^{\top}A_i(t)^{-1}b_i(t) + r_i(t)^2].
\end{align*}
\end{proposition}

Please refer to Appendix~\ref{subsec: app_disjoint} for detailed proof of Proposition~\ref{prop: posterior_disjoint}. We can also obtain the desired posterior distribution by applying variational Bayes to Lasso regression model (see Algorithm 2 in \cite{tran2021practical}). Algorithm~\ref{alg: update_disjoint} summarizes the posterior updating for the disjoint model. We now present the Thompson sampling algorithm for the disjoint model in Algorithm~\ref{alg: ts_disjoint}. At each round $t$, we generate a sample $\widetilde{\lambda}_i(t)$ from the distribution Gamma($C_i(t),D_i(t)$), set $\widetilde{\sigma}^2_i(t)=\frac{1}{\widetilde{\lambda}_i(t)}$, and generate a sample $\widetilde{\mu}_i(t)$ from the distribution $\mathcal{N}(A_i(t)^{-1}b_i(t), (\widetilde{\lambda}_i A_i(t))^{-1})$ for each arm $i$. Then we play the arm $i$ that maximizes 
{\small
\begin{align*}
    \widetilde{\operatorname{MV}}_i(t) = x_i(t)^{\top}\widetilde{\mu}_i(t)-\rho\widetilde{\sigma}^2_i(t).
\end{align*}
}

\begin{algorithm*}[!ht]
\SetAlgoLined
\SetKwInOut{Input}{input}\SetKwInOut{Output}{output}
\Input{prior parameters ($A_i(t),b_i(t),C_i(t),D_i(t)$), context $x_i(t)$, set of rounds that arm $i$ has been pulled during the first $t-1$ rounds $T_i(t)$; arm to play $a(t)$; and reward sample $r_{a(t)}(t)$.}
\Output{posterior parameters ($A_i(t+1),b_i(t+1),C_i(t+1),D_i(t+1)$) for each arm $i$.}
\For{$i = 1,2, \cdots, K$, $i \neq a(t)$}{$A_i(t+1)=A_i(t), b_i(t+1)=b_i(t)$,$C_i(t+1)=C_i(t), D_i(t+1)=D_i(t)$\;}
$A_{a(t)}(t+1)=A_{a(t)}(t)+x_{a(t)}(t) x_{a(t)}(t)^{\top}$\label{alg: posterior_update_A}\;
$b_{a(t)}(t+1)=b_{a(t)}(t) + x_{a(t)} r_{a(t)}(t)$\label{alg: posterior_update_b}\;
$C_{a(t)}(t+1) = C_{a(t)}(t) + \frac{1}{2}$\; 
$D_{a(t)}(t+1) = D_{a(t)}(t) + \frac{1}{2}[b_{a(t)}(t)^T A_{a(t)}(t)^{-1} b_{a(t)}(t) - b_{a(t)}(t+1)^{\top} A_{a(t)}(t+1)^{-1} b_{a(t)}(t+1) + r_{a(t)}(t)^2]$.
\caption{Posterior updating for the disjoint model at round $t$.}
\label{alg: update_disjoint}
\end{algorithm*}

\begin{algorithm*}[!ht]
\SetAlgoLined
\textbf{initialization}:\\
{pull arm $i$ once at round 0 and observe rewards $r_i(0)$\; 
$A_i(1)=\mathbf{I}_d + x_i(0)x_i(0)^{\top}, b_i(1)=x_i(0)r_i(0), C_i(1)=\frac{1}{2}, D_i(1)=\frac{1}{2}(r_i(0)^2-x_i(0)^{\top}A_i(1)^{-1}x_i(0)), T_i(1)=\{0\}$\;
}
\For{$t=1,2,\cdots,T$}
{
observe $K$ contexts $x_1(t),\cdots,x_K(t) \in \mathbb{R}^{d}$\;
\For{$i=1,2,\cdots,K$}
{
sample $\widetilde{\lambda}_i(t)$ from distribution Gamma($C_i(t),D_i(t)$), set $\widetilde{\sigma}^2_i(t)=\frac{1}{\widetilde{\lambda}_i(t)}$\label{alg: variance_sampling}\;
sample $\widetilde{\mu}_i(t)$ from distribution $\mathcal{N}\left(A_i(t)^{-1}b_i(t), (\widetilde{\lambda}_i(t) A_i(t))^{-1}\right)$\label{alg: mean_sampling}\;
set $\widetilde{\operatorname{MV}}_i(t) = x_i(t)^{\top}\widetilde{\mu}_i(t)-\rho\widetilde{\sigma}^2_i(t)$\;
}
play arm $a(t)=\arg\max_{i \in [K]} \widetilde{\operatorname{MV}}_i(t)$ with ties broken arbitrarily\;
observe reward $r_{a(t)}(t) \sim \nu_{a(t)}\left(x_{a(t)}(t)^{\top} \mu_{a(t)}, \sigma_{a(t)}^2\right)$\;
update ($A_i(t), b_i(t), C_i(t), D_i(t)$) according to Algorithm~\ref{alg: update_disjoint} for each arm $i$\;
set $T_{a(t)}(t+1) = T_{a(t)}(t) \bigcup \{t\}$.
}
\caption{Mean-variance Thompson sampling for the disjoint model (MVTS-D).}
\label{alg: ts_disjoint}
\end{algorithm*}

\section{Regret Analysis}
\label{sec: analysis}
In this section, we present our regret bounds and its derivation for a variant of the proposed MVTS-D algorithm. We first make the following assumptions.
\begin{assumption}~
\begin{itemize}
    \item[(i)] $\eta_i(t):=r_i(t)-x_i(t)^{\top}\mu_i$ is R-sub-Gaussian, i.e., 
    \begin{align*}
        \mathbb{E}[\lambda \exp(\eta_i(t))] \leq \exp(\frac{\lambda^2 R^2}{2}), \forall \lambda \in \mathbb{R}.
    \end{align*}
    \item[(ii)] $\eta_i(t)^2-\sigma_i^2$ is R-sub-Gaussian, i.e.,
    \begin{align*}
        \mathbb{E}[\lambda \exp(\eta_i(t)^2-\sigma_i^2)] \leq \exp(\frac{\lambda^2 R^2}{2}), \forall \lambda \in \mathbb{R}.
    \end{align*}
    \item[(iii)] $\left\|x_{i}(t)\right\| \leq 1,\|\mu\| \leq 1$, $\mid x_i(t)^{\top} \mu_i - x_j(t)^{\top} \mu_j \mid \leq 1$, $\mid \sigma_i^2 - \sigma_j^2 \mid \leq 1$, for $i, j \in [K], i \neq j$, for all $t$.
\end{itemize}
\label{ass: regret}
\end{assumption}

The first and second assumption in Assumption~\ref{ass: regret} are satisfied when the reward distribution is bounded. In case the positive constants $R$ in (i) and (ii) are different, we take $R$ to be the maximum of the two. The third assumption is required to make the regret bound scale-free, and are standard in the literature (see \cite{agrawal2013thompson}). The norms $\|\cdot\|$, unless stated otherwise, are $l_2$-norms. In case $\left\|x_{i}(t)\right\| \leq c_1,\|\mu\| \leq c_2$, $\mid x_i(t)^{\top} \mu_i - x_j(t)^{\top} \mu_j \mid \leq c_3$, $\mid \sigma_i^2 - \sigma_j^2 \mid \leq c_4$ for some constants $c_1,c_2,c_3,c_4>0$, our regret bound would increase by a factor of $c=\max\{c_1,c_2,c_3,c_4\}$. 

Due to the technical difficulty, instead of sampling the regret variance from the posterior Gamma distribution, we propose to sample the regret variance from a Gaussian distribution with a decaying variance term, where the mean of the Gaussian distribution corresponds to the mean of the Gamma distribution. Gaussian sampling enables us to derive the desired concentration and anti-concentration bounds, which are crucial in the regret analysis. Also, similar to \cite{zhu2020thompson}, we sample the regret mean and regret variance from different distributions independently. We summarize this variant of the MVTS algorithm in Algorithm~\ref{alg: MVTS_DN} (see Appendix~\ref{subsec: app_referenced_algorithms} for full algorithm) and name it as MVTS-DN, since it samples the regret variance from a normal distribution. Compared to Algorithm~\ref{alg: ts_disjoint}, Algorithm~\ref{alg: MVTS_DN}
\begin{itemize}
    \item replaces Line~\ref{alg: variance_sampling} by: sample $\widetilde{\sigma}_i^2(t)$ from distribution $\mathcal{N}\left(\frac{D_i(t)}{C_i(t)}, \frac{u^2}{\mid T_i(t)\mid}\right)$;
    \item replaces Line~\ref{alg: mean_sampling} by: sample $\widetilde{\mu}_i(t)$ from distribution $\mathcal{N}\left(A_i(t)^{-1}b_i(t), v^2 A_i(t)^{-1}\right)$.
\end{itemize}

In Algorithm~\ref{alg: MVTS_DN}, $v=R\sqrt{\frac{4}{\epsilon}d\ln\frac{4K}{\delta}}$, $u=8R^2 d\ln\frac{4K}{\delta}\sqrt{\frac{1}{\epsilon}}$, where $0<\delta<1$ is the parameter for confidence level ($1-\delta$), and $0<\epsilon<\frac{1}{2}$ is the parameter that controls the prior variance in the sampling process. A smaller $\epsilon$ leads to a larger prior variance, which encourages more exploration. We first show the main result of the theoretical analysis and discuss the proof of the result later.

\begin{theorem}
Suppose Assumption~\ref{ass: regret} holds. For the contextual MAB problem with $T$ rounds, $K$ arms, $d$-dimensional contexts and linear reward under the mean-variance criterion, the MVTS-DN algorithm achieves a total regret of $O((1+\rho+\frac{1}{\rho}) d \ln T \ln \frac{K}{\delta}\sqrt{d K  T^{1+2\epsilon} \ln \frac{K}{\delta} \frac{1}{\epsilon}})$ that holds with probability $1-\delta$, for any $0<\epsilon<\frac{1}{2}$, $0<\delta<1$.
\label{thm: regret}
\end{theorem}

\begin{remark}
Treating all parameters as constants except the number of rounds $T$ and $\epsilon$, we achieve a regret bound of ${O} (\sqrt{T^{1+2\epsilon}}\ln T)$, which essentially is the same as that in \cite{agrawal2013thompson}. However, when considering the number of arms $K$, confidence level $\delta$ and dimension $d$, compared with \cite{agrawal2013thompson}, here the regret bound has additional terms $\ln \frac{K}{\delta}$ and $d$. This is caused by controlling the estimation error of regret variance, which is more difficult than that of regret mean. This can be seen in Lemma \ref{lem: 1} and \ref{lem: 2}, in which we obtain the constants $l(T)$ and $h(T)$ with different orders of $\ln \frac{K}{\delta}$ and $d$. 
\end{remark}

To prove Theorem \ref{thm: regret}, we follow a similar approach as in \cite{agrawal2013thompson}. Compared with the risk-neutral case in \cite{agrawal2013thompson}, the main difficulty of regret analysis for this risk-averse case arises in the estimation error control of the variance, which appears in our mean-variance objective for arm selection. This difficulty is overcome by sampling the regret variance $\widetilde{\sigma}_i^2(t)$ from a normal distribution instead of the Gamma distribution, as we have argued in the beginning of this section. For ease of exposition, we introduce and summarize some notations below that are relevant to the proofs.

Let the mean parameter estimate be $\hat{\mu}_i(t)=A_i(t)^{-1}b_i(t)$, which is the weighted average of the historical rewards for arm $i$ up to rounds $t-1$. Let the standard deviation of the estimate $x_i(t)^{\top} \hat{\mu}_i(t)$ be $s_{i}(t)=\sqrt{x_{i}(t)^{T} A_i(t)^{-1} x_{i}(t)}$. In the variance sampling, $C_i(t)=\frac{1}{2}\mid T_i(t)\mid$, $D_i(t)=\frac{1}{2}[\sum_{s \in T_i(t)} r_i(s)^2 - b_i(t)^{\top}A_i(t)^{-1}b_i(t)]$. Let the variance parameter estimate be $\hat{\sigma}_i^2(t)=\frac{D_i(t)}{C_i(t)}=\frac{1}{\mid T_i(t)\mid}(\sum_{s \in T_i(t)} r_i(s)^2 - b_i(t)^{\top}A_i(t)^{-1}b_i(t))$.

\begin{definition}
Define the following constants in terms of $T$:
$$
\ell(T)=R\sqrt{d\ln T \ln\frac{4K}{\delta}}+1,
$$
$$
h(T)=4R^2 d\ln\frac{4K}{\delta}\sqrt{\ln T},
$$
$$
g(T)=\sqrt{4 d \ln T \sqrt{Kd}}\cdot v+\ell(T),
$$
$$
q(T)=u\sqrt{2\ln T}+h(T).
$$
\end{definition}

These constants are used throughout the proof.

\begin{definition} \label{def: saturated} The saturated set $S(t)$,
\begin{align*}
    S(t):=\Big\{i \in [K&]: g(T)s_i(t)+\rho q(T) \frac{1}{\sqrt{\mid T_i(t) \mid}} \\
    & \leq \ell(T) s_{a^{*}(t)}(t) + \rho h(T) \frac{1}{\sqrt{\mid T_i(t) \mid}}\Big\}.
\end{align*}
An arm $i$ is called saturated at round $t$ if $i \in S(t)$, and unsaturated if $i \notin S(t)$.
\end{definition}

For an saturated arm $i$, the standard deviation $s_i(t)$ is small and the number of pulls $\mid T_i(t) \mid$ is large. Hence, the estimates of the mean and variance parameters constructed using the previous rewards are quite accurate. The algorithm can easily tell whether it is optimal arm or not. At last, let the filtration $\mathcal{F}_{t-1}$ be the $\sigma$-algebra generated by $\mathcal{H}_{t-1} \bigcup \{x_i(t)\}_{i \in [K]}$. 
 
\subsection{Proof Outline}
We present the proof outline here. We first derive confidence bands for mean and variance parameter estimates $\hat{\mu}_i(t)$, $\hat{\sigma}_i^2(t)$, for all $i$. Then we derive confidence bands for sampled regret mean and sampled regret variance $\widetilde{\mu}_i(t)$ and $\widetilde{\sigma}_i^2(t)$, for all $i$. Using these bands and the triangle inequality, we have $ MV_{a^*(t)}(t) - MV_{a(t)}(t) \le \widetilde{MV}_{a^*(t)}(t) - \widetilde{MV}_{a(t)}(t) + g(T)(s_{a^*(t)}(t) + s_{a(t)}(t))+ \rho q(T)\frac{1}{\sqrt{\mid T_{a^*(t)}(t)\mid}}$. Since $a(t)$ is the arm with largest $\widetilde{MV}$, the regret at any time $t$ can be bounded by $g(T)(s_{a^*(t)}(t) + s_{a(t)}(t))+ \rho q(T)(\frac{1}{\sqrt{\mid T_{a^*(t)}(t)\mid}} +\frac{1}{\sqrt{\mid T_{a(t)}(t)\mid}})$, where the four terms represent the confidence bands for arm $a^*(t)$ and $a(t)$. Then, we can bound the total regret if we can bound $\sum_{t=1}^T s_{a(t)}(t) $, $\sum_{t=1}^T \frac{1}{\sqrt{\mid T_{a(t)}(t)}}$, $\sum_{t=1}^T s_{a^*(t)}(t) $ and $\sum_{t=1}^T \frac{1}{\sqrt{\mid T_{a^*(t)}(t)}}$, respectively. For the first two terms, we have $\sum_{t=1}^T s_{a(t)}(t) = O(\sqrt{Td\ln T })$ and $\sum_{t=1}^T \frac{1}{\sqrt{\mid T_{a(t)}(t)}} = O(\sqrt{T\ln T })$. The challenge is left to bound  $\sum_{t=1}^T s_{a^*(t)}(t) $ and $\sum_{t=1}^T \frac{1}{\sqrt{\mid T_{a^*(t)}(t)}}$. 

For this purpose, we define the saturated and unsaturated arms at any time as in Definition~\ref{def: saturated}. Then, if an arm $a(t) \not \in S(t)$ is played at time $t$, we can bound $s_{a^*(t)}(t) $ and $\frac{1}{\sqrt{\mid T_{a^*(t)}(t)}} $ by $s_{a(t)}(t) $ and $\frac{1}{\sqrt{\mid T_{a(t)}(t)}} $ multiplied by some factors according to the definition of unsaturated arms. For saturated arms in $S(t)$, we bound the probability of playing such arms at any time $t$ by the probability of playing the optimal arm at time $t$, $a^*(t)$ multiplied by some factor, given the filtration $\mathcal{F}_{t-1}$. This is helpful since again we can shift  those terms indexed with $a^*(t)$ to terms indexed with $a(t)$. 

With all the observations, we establish a super-martingale difference, $Y_t$, with respect to the regret as shown in Lemma \ref{lem: 8}. Applying the Azuma-Hoeffding inequality for martingales and along with $\sum_{t=1}^T s_{a(t)}(t) = O(\sqrt{Td\ln T })$ and $\sum_{t=1}^T \frac{1}{\sqrt{\mid T_{a(t)}(t)}} = O(\sqrt{T\ln T })$, we obtain the high probability regret bound in Theorem \ref{thm: regret}.

\subsection{Formal Proof of Theorem \ref{thm: regret}}
\begin{lemma}[\cite{abbasi2011improved}, Theorem 1]\label{lem: self-normalized}
Let $\left\{\mathcal{F}_{t}\right\}_{t=0}^{\infty}$ be a filtration. Let $\left\{\eta_{t}\right\}_{t=1}^{\infty}$ be a real-valued stochastic process such that $\eta_{t}$ is $\mathcal{F}_{t}$-measurable and $\eta_{t}$ is conditionally $R$-sub-Gaussian for some $R \geq 0$.
Let $\left\{m_{t}\right\}_{t=1}^{\infty}$ be $\mathbb{R}^{d}$-valued stochastic process such that $m_{t}$ is $\mathcal{F}_{t-1}$-measurable. For any $t \geq 0$, define
$$
\bar{M}_{t}=I_d+\sum_{s=1}^{t} m_{s} m_{s}^{\top} \quad \xi_{t}=\sum_{s=1}^{t} \eta_{s} m_{s} .
$$
Then, for any $\delta>0$, with probability at least $1-\delta$, for all $t \geq 0$,
$$
\left\|\xi_{t}\right\|_{\bar{M}_{t}^{-1}}^{2} \leq 2 R^{2} \log \left(\frac{\operatorname{det}\left(\bar{M}_{t}\right)^{\frac{1}{2}} }{\delta}\right),
$$
where  $\left\|\xi_{t}\right\|_{\bar{M}_{t}^{-1}}=\sqrt{\xi_{t}^{T} \bar{M}_{t}^{-1} \xi_{t}}$.
 \end{lemma} 
The first two lemmas, Lemma \ref{lem: 1} and Lemma \ref{lem: 2}, upper bound the probability of estimation error of mean and variance around their true value. 

\begin{lemma}[\cite{agrawal2013thompson}, Lemma 1]\label{lem: 1}
Define $E^{\mu}(t)$ as the event that $x_i(t)^{\top}\hat{\mu}_i(t)$ is concentrated around its mean for any arm $i$, i.e.,
$$
E^{\mu}(t):=\{\mid x_i(t)^{\top}\hat{\mu}_i(t) - x_i(t)^{\top} \mu_i(t) \mid \leq \ell(T) s_i(t), \forall i \in [K] \}.
$$
Then with probability at least $1 - \frac{\delta}{4}$, $E^{\mu}(t)$ holds true for all $t$ and $0<\delta<1$.
\end{lemma}

\begin{lemma}\label{lem: 2}
Define $E^{\sigma}(t)$ as the event that $\hat{\sigma}_i^2(t)$ is concentrated around the true variance $\sigma_i^2$ for any arm $i$, i.e.,
$$
E^{\sigma}(t):=\{\mid \hat{\sigma}_i^2(t) - \sigma_i^2 \mid \leq h(T) \frac{1}{\sqrt{\mid T_i(t) \mid}}, \forall i \in [K] \}.
$$
Then with probability at least $1-\frac{\delta}{4}$, $E^{\sigma}(t)$ holds true for all $t$ and $0<\delta<1$.
\end{lemma}
\begin{proof}
Recall that $b_i(t)=\sum_{s \in T_i(t)} x_i(s)r_i(s)$. We have
\begin{align}
    & \hat{\sigma}_i^2(t) - \sigma_i^2 \nonumber\\
     =& \frac{1}{\mid T_i(t) \mid} \Big[\sum_{s \in T_i(t)} r_i(s)^2 - b_i(t)^{\top}A_i(t)^{-1}b_i(t)\Big]-\sigma_i^2\nonumber\\
     =&\frac{1}{\mid T_i(t) \mid} \Big[\sum_{s \in T_i(t)} r_i(s)^2 - \sum_{s \in T_i(t)} x_i(s)^{\top}r_i(s) A_i(t)^{-1}\nonumber\\
    & ~~~~~~~~~~~~~~~~\sum_{s \in T_i(t)} x_i(s)r_i(s) \Big]-\sigma_i^2\nonumber\\
     =& \frac{1}{\mid T_i(t) \mid} \Big[\sum_{s \in T_i(t)}(r_i(s)-x_i(s)^{\top}\mu_i)^2 - \sigma_i^2 \Big]\nonumber\\
    & + \frac{1}{\mid T_i(t) \mid} \Big[2 \sum_{s \in T_i(t)} r_i(s)x_i(s)^{\top}\mu_i - \sum_{s \in T_i(t)} (x_i(s)^{\top}\mu_i)^2\Big]\nonumber\\
    & - \frac{1}{\mid T_i(t) \mid} \Big[\sum_{s \in T_i(t)} x_i(s)(r_i(s)-x_i(s)^{\top}\mu_i)\Big]^{\top}A_i(t)^{-1}\nonumber\\
    & ~~~~~~~~~~~~~~\Big[\sum_{s \in T_i(t)} x_i(s)(r_i(s)-x_i(s)^{\top}\mu_i)\Big]\nonumber\\
    & - \frac{2}{\mid T_i(t) \mid} \Big[\mu_i^{\top} \sum_{s \in T_i(t)} x_i(s)x_i(s)^{\top}A_i(t)^{-1}\nonumber \\
    &~~~~~~~~~~~~~~~~~~~\sum_{s \in T_i(t)} x_i(s)r_i(s)\Big]\label{eq: lemma2_1}\\
    & + \frac{1}{\mid T_i(t) \mid}\Big[\mu_i^{\top} \sum_{s \in T_i(t)}x_i(s)x_i(s)^{\top}A_i(t)^{-1}\nonumber\\
    &~~~~~~~~~~~~~~~~~~~\sum_{s \in T_i(t)}x_i(s)x_i(s)^{\top}\mu_i \Big].\label{eq: lemma2_2}
\end{align}
Recall that $A_i(t)=\mathbf{I}_d + \sum_{s \in T_i(t)}x_i(s)x_i(s)^{\top}$. Then
\begin{align*}
    \eqref{eq: lemma2_1}&=\frac{2}{\mid T_i(t)\mid}\Big[\mu_i^{\top}\sum_{s \in T_i(t)}x_i(s)r_i(s)-\mu_i^{\top}A_i(t)^{-1}\\
    & ~~~~~~~~~~~~~~~~~~~\sum_{s \in T_i(t)}x_i(s)r_i(s) \Big].
\end{align*}
\begin{align*}
    \eqref{eq: lemma2_2}&=\frac{1}{\mid T_i(t)\mid}\Big[\mu_i^{\top}\sum_{s \in T_i(t)}(x_i(s)^{\top}\mu_i)^2 -\mu_i^{\top}A_i(t)^{-1}\\
    & ~~~~~~~~~~~~~~~~~~~\sum_{s \in T_i(t)}x_i(s)x_i(s)^{\top}\mu_i \Big].
\end{align*}
Then we obtain
\begin{align}
    & \hat{\sigma}_i^2(t) - \sigma_i^2 \nonumber\\
     =& \frac{1}{\mid T_i(t) \mid} \Big[\sum_{s \in T_i(t)}(r_i(s)-x_i(s)^{\top}\mu_i)^2 - \sigma_i^2 \Big]\label{eq: lemma2_3}\\
    & - \frac{1}{\mid T_i(t) \mid} \Big[\sum_{s \in T_i(t)} x_i(s)(r_i(s)-x_i(s)^{\top}\mu_i)\Big]^{\top}A_i(t)^{-1}\nonumber\\
    & ~~~~~~~~~~~~~~\Big[\sum_{s \in T_i(t)} x_i(s)(r_i(s)-x_i(s)^{\top}\mu_i)\Big]\label{eq: lemma2_4}\\
    & + \frac{2}{\mid T_i(t) \mid} \Big[\mu_i^{\top}A_i(t)^{-1}\sum_{s \in T_i(t)} x_i(s)(r_i(s)-x_i(s)^{\top}\mu_i)\Big]\label{eq: lemma2_5}\\
    & + \frac{1}{\mid T_i(t) \mid} \mu_i^{\top}\mu_i - \frac{1}{\mid T_i(t) \mid}\mu_i^{\top}A_i(t)^{-1}\mu_i\nonumber.
\end{align}
To bound \eqref{eq: lemma2_3}, we apply Lemma \ref{lem: self-normalized}. Let the filtration $\mathcal{F}_{t-1}'$ be the $\sigma$-algebra generated by $\mathcal{H}_{t-1} \bigcup \{x_i(t)\}_{i \in [K]} \bigcup a(t)$. Let
$$
\eta_i(t)=\left\{ \begin{array}{rcl} (r_i(t)-x_i(t)^{\top}\mu_i)^2-\sigma_i^2 & \mbox{if} & a(t)=i \\ 0 & \mbox{if} & a(t) \neq i
\end{array}\right. .
$$
Then $\eta_i(t)$ is $\mathcal{F}_{t}'$-measurable. Let
$$
m_i(t)=\left\{ \begin{array}{rcl} 1 & \mbox{if} & a(t)=i \\ 0 & \mbox{if} & a(t) \neq i
\end{array}\right. .
$$
Then $m_i(t)$ is $\mathcal{F}_{t-1}'$-measurable. Let
\begin{align*}
    \zeta_i(t) & =\sum_{s=1}^{\top}m_i(t)\eta_i(t)\\
    & = \sum_{s \in T_i(t)}[(r_i(s)-x_i(s)^{\top}\mu_i)^2-\sigma_i^2].
\end{align*}

Lemma \ref{lem: self-normalized} implies that with probability at least $1-\delta'$,
$$
\frac{1}{\sqrt{\mid T_i(t) \mid}}\mid \zeta_i(t)\mid \leq R\sqrt{\ln\frac{\mid T_i(t)\mid}{\delta'^2}},
$$
therefore we have 
$$
\mid \eqref{eq: lemma2_3} \mid \leq \frac{1}{\sqrt{\mid T_i(t)\mid}}R\sqrt{\ln\frac{\mid T_i(t)\mid}{\delta'^2}}.
$$
\eqref{eq: lemma2_4} is bounded similarly using Lemma \ref{lem: self-normalized}. Let
$$
\eta_i(t)=\left\{ \begin{array}{rcl} r_i(t)-x_i(t)^{\top}\mu_i & \mbox{if} & a(t)=i \\ 0 & \mbox{if} & a(t) \neq i
\end{array}\right. ,
$$
$$
m_i(t)=\left\{ \begin{array}{rcl} x_i(t) & \mbox{if} & a(t)=i \\ 0 & \mbox{if} & a(t) \neq i
\end{array}\right. ,
$$
$$
\xi_i(t)=\sum_{s=1}^{t}m_i(t)\eta_i(t)=\sum_{s \in T_i(t)} x_i(s)(r_i(s)-x_i(t)^{\top}\mu_i).
$$
Note that $\operatorname{det}\left({A}_{i}(t)\right)  \leq (T_i(t))^d$. For $d \geq 2$, Lemma \ref{lem: self-normalized} implies that with probability at least $1-\delta'$,
\begin{align*}
    \mid \eqref{eq: lemma2_4} \mid & = \frac{1}{\mid T_i(t) \mid}\left\|\xi_i(t)\right\|^2_{A_{i}(t)^{-1}}\\
    & \leq \frac{1}{\mid T_i(t) \mid} R^2 d \ln \left(\frac{\mid T_i(t)\mid}{\delta^{\prime}}\right).
\end{align*}
Similarly, we have 
\begin{align*}
    \mid \eqref{eq: lemma2_5} \mid & \leq \frac{2}{\mid T_i(t)\mid} \left\|\mu_{i}(t)\right\|^2_{A_{i}(t)^{-1}} \left\|\xi_{i}(t)\right\|^2_{A_{i}(t)^{-1}}\\
    & \leq \frac{2}{\mid T_i(t)\mid} R\sqrt{d\ln\frac{\mid T_i(t)\mid}{\delta'}}.
\end{align*}
For the last two terms, we have
$$
\mid \frac{1}{\mid T_i(t)\mid} \mu_i^{\top}\mu_i \mid \leq \frac{1}{\mid T_i(t)\mid},
$$
$$
\mid \frac{1}{\mid T_i(t)\mid} \mu_i^{\top}A_i(t)^{-1}\mu_i \mid \leq \frac{1}{\mid T_i(t)\mid}.
$$
Assume $R\geq1$. Then with probability at least $1-2\delta'$, we have
\begin{align*}
    & \mid \hat{\sigma}_i^2(t) - \sigma_i^2 \mid \\
     \leq& \frac{1}{\sqrt{\mid T_i(t) \mid}} R \sqrt{\ln\frac{\mid T_i(t)\mid}{\delta'^2}} + \frac{1}{\mid T_i(t)\mid}R^2d\ln\frac{\mid T_i(t)\mid}{\delta'}\\
    & +\frac{2}{\mid T_i(t)\mid}R\sqrt{d\ln\frac{\mid T_i(t)\mid}{\delta'}}+\frac{2}{\mid T_i(t)\mid}\\
    \leq & \frac{1}{\sqrt{\mid T_i(t)\mid}} 4R^2d\ln\frac{1}{\delta'}\sqrt{\ln\mid T_i(t)\mid}.
\end{align*}
Taking $\delta'=\frac{\delta}{4K}$, we have
\begin{align*}
     & \mid \hat{\sigma}_i^2(t) - \sigma_i^2 \mid \\
      \leq&  \frac{1}{\sqrt{\mid T_i(t)\mid}} 4R^2d\ln\frac{4K}{\delta}\sqrt{\ln\mid T_i(t)\mid}\\
      \leq&  \frac{1}{\sqrt{\mid T_i(t)\mid}} 4R^2d\ln\frac{4K}{\delta}\sqrt{\ln\mid T\mid}\\
      :=& h(T) \frac{1}{\sqrt{\mid T_i(t)\mid}}.
\end{align*}
Then with probability at least $1-\frac{\delta}{4K}$, $\mid \hat{\sigma}_i^2(t) - \sigma_i^2 \mid  \leq h(T) \frac{1}{\sqrt{\mid T_i(t)\mid}}$ holds $\forall t\geq 1$. Using a union bound we obtain with probability at least $1-\frac{\delta}{4}$, $E^\sigma(t)$ holds $\forall t\geq1$.
\end{proof}

Lemma \ref{lem: 3} and Lemma \ref{lem: 4} provide concentration bounds for the posterior samples of regret mean and regret variance around their estimates, respectively.
\begin{lemma}[\cite{agrawal2013thompson}, Lemma 1]\label{lem: 3}
Define $E^{\tilde{\mu}}(t)$ as the event that $x_i(t)^{\top}\widetilde{\mu}_i(t)$ is concentrated around $x_i(t)^{\top}\hat{\mu}_i(t)$ for any arm $i$, i.e.,
\begin{align*}
    E^{\tilde{\mu}}(t):=\Big\{ & \mid x_i(t)^{\top}\widetilde{\mu}_i(t) - x_i(t)^{\top} \hat{\mu}_i(t) \mid \\
    & \leq \sqrt{4d\ln T \sqrt{Kd}} \cdot v \cdot s_i(t), \forall i \in [K] \Big\}.
\end{align*}Then  $\mathbb{P}(E^{\tilde{\mu}}(t) \mid \mathcal{F}_{t-1}) \ge 1 - \frac{1}{T^2}$.
\end{lemma}

\begin{lemma}\label{lem: 4}
Define $E^{\widetilde{\sigma}}(t)$ as the event that $\widetilde{\sigma}_i^2(t)$ is concentrated around $\hat{\sigma}_i^2(t)$ for any arm $i$, i.e.,
$$
E^{\widetilde{\sigma}}(t):=\{\mid \widetilde{\sigma}_i^2(t) - \hat{\sigma}_i^2(t) \mid \leq 2\sqrt{\ln T\sqrt{K}}u \frac{1}{\sqrt{\mid T_i(t) \mid}}, \forall i \in [K] \}.
$$
Then $\mathbb{P}(E^{\widetilde{\sigma}}(t)  \mid \mathcal{F}_{t-1}) \ge 1 - \frac{1}{T^2}$.
\end{lemma}
\begin{proof}
A direct application of Lemma 5 in \cite{agrawal2013thompson} gives:
\begin{align*}
    & \mathbb{P}\Big(\mid \frac{\sqrt{\mid T_i(t) \mid}}{u} (\widetilde{\sigma}_i^2(t) - \hat{\sigma}_i^2(t))\mid \geq 2\sqrt{\ln T\sqrt{K}}\mid \mathcal{F}_{t-1} \Big)\\
    & \leq \frac{1}{\sqrt{\pi}}\cdot \frac{1}{2\sqrt{\ln T\sqrt{K}}} \exp(-2\ln T\sqrt{K})\\
    & \leq \frac{1}{KT^2}.
\end{align*}
Hence $E^{\widetilde{\sigma}}(t)$ holds with probability at least $1-K \cdot \frac{1}{KT^2} = 1 - \frac{1}{T^2}$.
\end{proof}

With Lemma \ref{lem: 1}-\ref{lem: 4}, we can derive the concentration bounds for $\tilde{\mu}_i$ and $\tilde{\sigma}^2_i$ around the true value $\mu_i$ and $\sigma_i^2$, which is useful to bound the regret in terms of $s_{a(t)}(t)$, $s_{a^*(t)}(t)$, $\frac{1}{\sqrt{\mid T_{a(t)}(t)} \mid}$ and $\frac{1}{\sqrt{\mid T_{a^*(t)}(t)} \mid}$. 

The next step is to bound the probability of pulling an arm in the set $S(t)$ as shown in Lemma \ref{lem: 7} by the probability of pulling an optimal arm. To prove lemma \ref{lem: 7}, we first present Lemma \ref{lem: 5} and Lemma \ref{lem: 6}, which lower bound the probability of sampling $\tilde{\mu}_i$ such that $ x_i(t)^\intercal \tilde{\mu}_i(t)$ exceeding $x_i^\intercal(t)\mu_i$ by $\ell(t)s_i(t)$ and sampling $\tilde{\sigma}_i^2$ less than $\sigma_i^2 - h(T)\frac{1}{\sqrt{\mid T_i(t) \mid}}$. 

\begin{lemma}\label{lem: 5}
Conditioned on $E^{\mu}(t)$, we have
\begin{align*}
    \mathbb{P}\Big(x_i(t)^{\top}\widetilde{\mu}_i(t) & \geq x_i(t)^{\top}\mu_i + \ell(T) s_i(t) \mid \mathcal{F}_{t-1} \Big) \\
    & \geq \frac{1}{2\sqrt{\pi \epsilon \ln T \cdot T^{\epsilon}}}.
\end{align*}
\end{lemma}

\begin{proof}
Given the event $E^{\mu}(t)$, we have
$$
\mid x_i(t)^{\top}\hat{\mu}_i(t) - x_i(t)^{\top} \mu_i(t) \mid \leq \ell(T) s_i(t).
$$
Since $x_i(t)^{\top}\widetilde{\mu}_i(t)$ is a Gaussian random variable that has mean $x_i(t)^{\top}\hat{\mu}_i(t)$ and standard deviation $v s_i(t)$. Using the anti-concentration inequality in Lemma 5 in \cite{agrawal2013thompson}, we have
\begin{align*}
    & \mathbb{P}\Big(x_i(t)^{\top}\widetilde{\mu}_i(t) \geq x_i(t)^{\top}\mu_i + \ell(T) s_i(t) \mid \mathcal{F}_{t-1} \Big) \nonumber\\
    & = \mathbb{P}\Big(\frac{x_i(t)^{\top}\widetilde{\mu}_i(t) - x_i(t)^{\top}\hat{\mu}_i(t)}{v s_i(t)} \geq \nonumber\\
    & ~~~~~~~~\frac{x_i(t)^{\top}\mu_i - x_i(t)^{\top}\hat{\mu}_i(t) + \ell(T)s_i(t)}{v s_i(t)} \mid \mathcal{F}_{t-1} \Big) \nonumber\\
    & \geq \mathbb{P}\Big(\frac{x_i(t)^{\top}\widetilde{\mu}_i(t) - x_i(t)^{\top}\hat{\mu}_i(t)}{v s_i(t)} \geq Z_t \mid \mathcal{F}_{t-1} \Big)\nonumber\\
    & \geq \frac{1}{\sqrt{\pi}} \frac{1}{Z_t + 1 / Z_t} \exp(-\frac{Z_t^2}{2}).
\end{align*}
where
\begin{align*}
    \mid Z_t \mid & = \sqrt{\epsilon \ln T} \\
    & = \frac{2 \ell(T)}{v} \\
    & \geq \mid \frac{x_i(t)^{\top}\mu_i - x_i(t)^{\top}\hat{\mu}_i(t) + \ell(T)s_i(t)}{v s_i(t)} \mid.
\end{align*}
Therefore, we have
\begin{align*}
    & \mathbb{P}\Big(x_i(t)^{\top}\widetilde{\mu}_i(t) \geq x_i(t)^{\top}\mu_i + \ell(T) s_i(t) \mid \mathcal{F}_{t-1} \Big) \nonumber\\
    & \geq \frac{1}{\sqrt{\pi}} \frac{1}{\sqrt{\epsilon \ln T} + 1 / \sqrt{\epsilon \ln T}} \exp(-\frac{\epsilon \ln T}{2}).
\end{align*}
Without loss of generality, we can set $\epsilon \ln T \geq 1$ for large $T$. Thus
\begin{align*}
    & \mathbb{P}\Big(x_i(t)^{\top}\widetilde{\mu}_i(t) \geq x_i(t)^{\top}\mu_i + \ell(T) s_i(t) \mid \mathcal{F}_{t-1} \Big) \nonumber\\
    & \geq \frac{1}{2\sqrt{\pi \epsilon \ln T \cdot T^{\epsilon}}}.
\end{align*}
\end{proof}

\begin{lemma}\label{lem: 6}
Conditioned on $E^{\sigma}(t)$, we have
\begin{align*}
    \mathbb{P}\Big(\widetilde{\sigma}_i^2(t) & \leq \sigma_i^2 - h(T) \frac{1}{\sqrt{\mid T_i(t) \mid}} \mid \mathcal{F}_{t-1} \Big) \\
    & \geq \frac{1}{2\sqrt{\pi \epsilon \ln T \cdot T^{\epsilon}}}.
\end{align*}
\end{lemma}
\begin{proof}
The proof is similar to Lemma~\ref{lem: 5}.
\end{proof}

Define shorthand notations
$$
\omega(t)=x_{a^{*}(t)}(t)^{\top}\mu_{a^{*}(t)} - x_i(t)^{\top}\mu_i,
$$
$$
\Gamma_i(t)=\sigma_i^2 - \sigma_{a^{*}(t)}^2,
$$
$$
\Lambda_i(t)=x_i(t)^{\top}\mu_i - \rho \sigma_i^2.
$$
Replacing $\mu_i$, $\sigma_i$ by their estimates $\hat{\mu}_i(t)$, $\hat{\sigma}_i(t)$, we have corresponding shorthand notations for $\hat{\omega}(t)$, $\hat{\Gamma}_i(t)$ and $\hat{\Lambda}_i(t)$. Replacing $\hat{\mu}_i(t)$, $\hat{\sigma}_i(t)$ by their samples $\widetilde{\mu}_i(t)$, $\widetilde{\sigma}_i(t)$, we have corresponding shorthand notations for $\widetilde{\omega}(t)$, $\widetilde{\Gamma}_i(t)$ and $\widetilde{\Lambda}_i(t)$.

\begin{lemma}\label{lem: 7}
Given any filtration $\mathcal{F}_{t-1}$ such that event $E^{\mu}(t)$ and $E^{\sigma}(t)$ hold, we have
\begin{align*}
    \mathbb{P}\big(a(t) \in S(t) \mid \mathcal{F}_{t-1}\big) \leq \frac{1}{p}\mathbb{P}\big(a(t)=a^{*}(t)\mid\mathcal{F}_{t-1}\big) + \frac{2}{pT^2},
\end{align*}
where $p=\frac{1}{4\pi\epsilon\ln T \cdot T^{\epsilon}}$.
\end{lemma}

\begin{proof}
Note that the algorithm chooses arm $a^{*}(t)$ to pull at round $t$ if the following event happens:
$$
\widetilde{\omega}_j(t) + \rho \widetilde{\Gamma}_j(t) \geq 0, \forall j \neq a^*(t).
$$
Therefore, we have
\begin{align*}
    & \mathbb{P}\big(a(t)=a^{*}(t)\mid\mathcal{F}_{t-1}\big) \\
    & \geq \mathbb{P}\big(\widetilde{\omega}_j(t) + \rho \widetilde{\Gamma}_j(t) \geq 0, \forall j \neq a^*(t) \mid \mathcal{F}_{t-1} \big)\\
    & \geq \mathbb{P}\big(\exists i \in S(t): \widetilde{\Lambda}_{a^{*}(t)}(t) \geq \widetilde{\Lambda}_i(t) \\
    & ~~~~~~~~~~~~~ \widetilde{\Lambda}_i(t) \geq \widetilde{\Lambda}_j(t), \forall j \neq a^{*}(t) \mid \mathcal{F}_{t-1} \big)\\
    & \geq \mathbb{P}\big(\forall i \in S(t), \widetilde{\Lambda}_{a^{*}(t)}(t) \geq \Lambda_{a^{*}(t)}(t) + \\
    &~~~~~~~~~~~~~\ell(T)s_{a^{*}(t)}(t)+\rho \frac{h(T)}{\sqrt{\mid T_i(t)\mid}} \geq \widetilde{\Lambda}_i(t), \\
    & ~~~~~~~\exists i \in S(t), \widetilde{\Lambda}_i(t) \geq \widetilde{\Lambda}_j(t), \forall j \neq a^{*}(t) \mid \mathcal{F}_{t-1}\big)\\
    & \geq \mathbb{P} \big ( \widetilde{\Lambda}_{a^{*}(t)}(t) \geq \Lambda_{a^{*}(t)}(t) + \ell(T)s_{a^{*}(t)}(t) + \frac{\rho h(T)}{\sqrt{\mid T_{a^{*}(t)}(t) \mid}},\\
    & ~~~~~~~\exists i \in S(t), \widetilde{\Lambda}_i(t) \geq \widetilde{\Lambda}_j(t), \forall j \neq a^{*}(t) \mid \mathcal{F}_{t-1} \big) \\
    & ~~- \mathbb{P}\big(\{\forall i \in S(t), \Lambda_{a^{*}(t)}(t) + \ell(T) s_{a^{*}(t)}(t) \\
    & ~~~~~~~~~~~~~~~~~~~~~+ \frac{\rho h(T)}{\sqrt{\mid T_{a^{*}(t)}(t) \mid}} \geq \widetilde{\Lambda}_i(t) \}^{c}\big)\\
    & = \mathbb{P}\big(\widetilde{\Lambda}_{a^{*}(t)}(t) \geq \Lambda_{a^{*}(t)}(t) + \ell(T)s_{a^{*}(t)}(t) \\
    & ~~~~~~~~~~~~~~~~~~~~~+ \frac{\rho h(T)}{\sqrt{\mid T_{a^{*}(t)}(t) \mid}} \mid \mathcal{F}_{t-1} \big)\\
    & ~~\cdot \mathbb{P}\big(\exists i \in S(t), \widetilde{\Lambda}_i(t) \geq \widetilde{\Lambda}_j(t), \forall j \neq a^{*}(t) \mid \mathcal{F}_{t-1} \big)\\
    & ~~- \mathbb{P}\big(\exists i \in S(t), \widetilde{\Lambda}_i(t) \geq \Lambda_{a^{*}(t)}(t) + \ell(T)s_{a^{*}(t)}(t) \\
    & ~~~~~~~~~~~~~~~~~~~~~+ \frac{\rho h(T)}{\sqrt{\mid T_{a^{*}(t)}(t) \mid}} \mid \mathcal{F}_{t-1} \big).
\end{align*}
Since 
\begin{align*}
    & \mathbb{P}\big( \widetilde{\Lambda}_{a^{*}(t)}(t) \geq \Lambda_{a^{*}(t)}(t) + \ell(T) s_{a^{*}(t)}(t) \\
    & ~~~~~~~~~~~~~~~ + \frac{\rho h(T)}{\sqrt{\mid T_{a^{*}(t)}(t) \mid}} \mid \mathcal{F}_{t-1}\big)\\
    & \geq \mathbb{P} \big(x_{a^{*}(t)}(t)^{\top} \widetilde{\mu}_{a^{*}(t)} \geq x_{a^{*}(t)}(t)^{\top} \mu_{a^{*}(t)} \\
    & ~~~~~~~~~~~~~~~+ \ell(T)s_{a^{*}(t)}(t) \mid \mathcal{F}_{t-1}\big)\\
    & ~~\cdot \mathbb{P}\big(\widetilde{\sigma}^2_{a^{*}(t)}(t) \geq \sigma^2_{a^{*}(t)} + \frac{h(T)}{\sqrt{\mid T_{a^{*}(t)}(t) \mid}} \mid \mathcal{F}_{t-1}\big)\\
    & \geq \frac{1}{4\pi \epsilon \ln T \cdot T^{\epsilon}} := p.
\end{align*}
Also note that when $E^{\widetilde{\mu}}(t) \bigcap E^{\widetilde{\sigma}}(t)$ holds true, we have that $\forall i \in S(t)$,
\begin{align*}
    \widetilde{\Lambda}_i(t) & \leq \Lambda_i(t) + g(T) s_{i}(t) + \rho h(T) \frac{1}{\sqrt{\mid T_i(t) \mid}} \\
    & \leq \Lambda_{a^{*}(t)}(t) + \ell(T) s_{a^{*}(t)}(t) + \rho q(T) \frac{1}{\sqrt{\mid T_i(t) \mid}}.
\end{align*}
Hence, we have
\begin{align*}
    & \mathbb{P}\big(a(t) = a^{*}(t) \mid \mathcal{F}_{t-1}\big) \\
    & \geq p \cdot \mathbb{P}\big(\exists i \in S(t), \widetilde{\Lambda}_i(t) \geq \widetilde{\Lambda}_j(t), \\
    &~~~~~~~~~~~\forall j \neq a^{*}(t) \mid \mathcal{F}_{t-1} \big) -\frac{2}{T^2} \\
    & \geq p \cdot \mathbb{P}\big(a(t) \in S(t) \mid \mathcal{F}_{t-1} \big) - \frac{2}{T^2}.
\end{align*}
Finally, we have
\begin{align*}
    \mathbb{P}\big(a(t) \in S(t) \mid \mathcal{F}_{t-1}\big) \leq \frac{1}{p}\mathbb{P}\big(a(t)=a^{*}(t)\mid\mathcal{F}_{t-1}\big) + \frac{2}{pT^2}.
\end{align*}
\end{proof}

We construct a super-martingale with respect to the regret in Lemma \ref{lem: 8}, which is used to bound the total regret later using Azuma-Hoeffding inequality. 

\begin{lemma}\label{lem: 8}
Recall that the regret at round $t$ is $\Delta_{a(t)}(t)$. Denote by $\mathbf{1}\{\cdot\}$ the indicator function. Let
$\Delta_{a(t)}^{\prime}(t) = \Delta_{a(t)}(t) \cdot \mathbf{1}\{E^{\mu}(t)\} \cdot \mathbf{1}\{E^{\sigma}(t)\}$. Let 
\begin{align*}
    Y_t  =& \Delta_{a(t)}^{\prime}(t)\\
    & - s_{a(t)}(t)\big(g(T) + \frac{g(T)^2}{\ell(T)} + \frac{g(T) q(T)}{\rho h(T)}\big) \\
    & - \frac{1}{\sqrt{\mid T_{a(t)}(t) \mid}} \big(\rho q(T) + \rho \frac{g(T) q(T)}{\ell(T)} + \rho \frac{q(T)^2}{h(T)} \big) \\
    & - s_{a^{*}}(t) \frac{g(T)}{p} \mathbf{1}\{a(t)=a^{*}(t)\} \\
    & - \frac{1}{\sqrt{\mid T_{a^{*}(t)}(t) \mid}} \rho\frac{q(T)}{p}\mathbf{1}\{a(t)=a^{*}(t)\}\\
    & - \frac{g(T)+\rho q(T)}{p T^2} - (1+\rho)\frac{2}{T^2}.
\end{align*}
Then $\sum_{s=1}^{t} Y_s$ is a super-martingale process with respect to the filtration $\mathcal{F}_{t}$. 
\end{lemma}

\begin{proof}
To prove $\sum_{s=1}^{t} Y_s$ is a super-martingale process, we need to show that for all $1 \leq t \leq T$ and a given filtration $\mathcal{F}_{t-1}$, $\mathbb{E}[Y_t \mid \mathcal{F}_{t-1}] \leq 0$. Conditioned on $E^{\mu}(t)$ and $E^{\sigma}(t)$, if both $E^{\widetilde{\mu}}(t)$ and $E^{\widetilde{\sigma}}(t)$ hold true, we have
$$
\omega_i(t) \leq \widetilde{\omega}_i(t) + g(T)(s_{a^{*}(t)}(t) + s_{a(t)}(t)),
$$
$$
\Gamma_i(t) \leq \widetilde{\Gamma}_i(t) + \rho q(T)(\frac{1}{\sqrt{\mid T_i(t) \mid}} + \frac{1}{\sqrt{\mid T_{a^{*}(t)}(t) \mid}}).
$$
Observe that
\begin{align}
    & \mathbb{E}[\Delta_{a(t)}^{\prime}(t) \mid \mathcal{F}_{t-1}] \nonumber\\
     \leq& g(T) \mathbb{E}[s_{a(t)}(t) \mid \mathcal{F}_{t-1}] \nonumber\\
    & + q(T) \mathbb{E}[\frac{1}{\sqrt{\mid T_{a(t)}(t) \mid}} \mid \mathcal{F}_{t-1}] \nonumber\\
    & + g(T) s_{a^{*}(t)}(t) + \rho q(T)\frac{1}{\sqrt{\mid T_{a^{*}(t)}(t) \mid}}\nonumber\\
    & ~~~\cdot \mathbb{E}[\mathbf{1}\{a(t) \in S(t)\} + \mathbf{1}\{a(t) \notin S(t)\} \mid \mathcal{F}_{t-1}] \nonumber\\
    & + (1+\rho)(1 - \mathbb{P}(E^{\widetilde{\mu}})) + (1+\rho)(1 - \mathbb{P}(E^{\widetilde{\sigma}}))\nonumber\\
    \leq & g(T)\mathbb{E}[s_{a(t)}(t) \mid \mathcal{F}_{t-1}] + \rho q(T) \mathbb{E}[\frac{1}{\sqrt{\mid T_{a(t)}(t) \mid}} \mid \mathcal{F}_{t-1}] \nonumber\\
    & + \big(g(T) s_{a^{*}(t)}(t) + \rho q(T) \frac{1}{\sqrt{\mid T_{a^{*}(t)}(t) \mid}} \big) \nonumber\\
    & ~~\cdot \mathbb{P}(a(t) \in S(t) \mid \mathcal{F}_{t-1}) \label{eq: lemma8_1}\\
    & + \mathbb{E}\big[ \big( g(T) s_{a^{*}(t)}(t)  + \rho q(T) \frac{1}{\sqrt{\mid T_{a^{*}(t)}(t) \mid}} \big) \nonumber\\
    & ~~~ \cdot \mathbf{1} \{a(t) \notin S(t) \} \mid \mathcal{F}_{t-1} \big] \label{eq: lemma8_2}\\
    & + (1 + \rho)\frac{2}{T^2} \nonumber.
\end{align}
Note that for \eqref{eq: lemma8_1}, we have
\begin{align*}
    \eqref{eq: lemma8_1} & \leq \big( g(T) s_{a^{*}(t)}(t) + \rho q(T) \frac{1}{\sqrt{\mid T_{a^{*}(t)}(t) \mid}} \big) \\
    & ~~~~\cdot \big[ \frac{1}{p}\mathbb{P}(a(t)=a^{*}(t) \mid \mathcal{F}_{t-1}) + \frac{1}{p T^2} \big]\\
    & \leq s_{a^{*}(t)}(t) \frac{g(T)}{p} \mathbb{P}(a(t)=a^{*}(t) \mid \mathcal{F}_{t-1}) \\
    &~~~ + \rho \frac{q(T)}{p} \mathbb{P}(a(t) = a^{*}(t) \mid \mathcal{F}_{t-1}) \frac{1}{\sqrt{\mid T_{a^{*}(t)}(t) \mid}} \\
    &~~~ + \frac{g(T)}{p T^2} + \frac{\rho g(T)}{p T^2}.
\end{align*}
For \eqref{eq: lemma8_2}, notice that when $a(t) \notin S(t)$, we have
\begin{align*}
    & g(T) s_{a(t)}(t) + \rho q(T) \frac{1}{\sqrt{\mid T_{a(t)}(t) \mid}} \\
    & > \ell(T) s_{a^{*}(t)}(t) + \rho h(T) \frac{1}{\sqrt{\mid T_{a^{*}(t)}(t) \mid}}.
\end{align*}
Rearrange the above inequality, we have
\begin{align*}
    s_{a^{*}(t)}(t) < \frac{g(T)}{\ell(T)} s_{a(t)}(t) + \frac{\rho q(T)}{\ell(T)} \frac{1}{\sqrt{\mid T_{a(t)}(t) \mid}}.
\end{align*}
Also note that
\begin{align*}
    \frac{1}{\sqrt{\mid T_{a^{*}(t)}(t) \mid}} < \frac{g(T)}{\rho h(T)}s_{a(t)}(t) + \frac{q(T)}{h(T)}\frac{1}{\sqrt{\mid T_{a(t)}(t) \mid}}.
\end{align*}
Hence
\begin{align*}
     \eqref{eq: lemma8_2} \leq& \mathbb{E}\Big[ \big( \frac{g(T)^2}{\ell(T)} + \frac{g(T) q(T)}{\rho h(T)} \big) s_{a(t)}(t) \\
    & + \big( \frac{\rho g(T) q(T)}{\ell(T)} + \frac{\rho q(T)^2}{h(T)} \big) \frac{1}{\sqrt{\mid T_{a(t)}(t) \mid}} \mid \mathcal{F}_{t-1}\Big].
\end{align*}
Putting all these together, we have $\mathcal{F}_{t-1}$, $\mathbb{E}[Y_t \mid \mathcal{F}_{t-1}] \leq 0$, thus $\sum_{s=1}^t Y_s$ is a super-martingale process.
\end{proof}

Now we start to prove the main result Theorem~\ref{thm: regret}.

\begin{proof}
First observe that we can bound the absolute value of $Y_t$ by $5 (1 + \rho + \frac{1}{\rho})\frac{q(T)^2}{\ell(T)}$. Therefore, by the Azuma-Hoeffding inequality, we have 
\begin{align*}
    \mathbb{P}(\sum_{t=1}^{T} Y_t \geq w) \leq \exp\Big(-\frac{w^2 \ell(T)^2}{(5(1+\rho+1/\rho))^2 q(T)^4}\Big):=\frac{\delta}{2}.
\end{align*}
Thus we set $w=5(1+\rho+1/\rho) \frac{q(T)^2}{\ell(T)} \sqrt{2 T \ln \frac{2}{\delta}}$. Then with probability at least $1-\frac{\delta}{2}$, we have
\begin{align*}
    & \sum_{t=1}^{T} \Delta^{\prime}_{a(t)}(t) \\
     \leq &\big( g(T) + \frac{g(T)^2}{\ell(T)} + \frac{g(T) q(T)}{\rho h(T)} + \frac{g(T)}{p} \big) \sum_{t=1}^{T} s_{a(t)}(t) \\
    & + \rho \big( g(T) + \frac{g(T) q(T)}{\ell(T)} + \frac{g(T)^2}{h(T)} + \frac{q(T)}{p}\big) \sum_{t=1}^{T}\frac{1}{\sqrt{\mid T_{a^{*}(t)}(t) \mid}} \\
    & + \frac{g(T) + \rho q(T)}{p T} + 5(1 + \rho + 1/\rho) \frac{q(T)^2}{\ell(T)}\sqrt{2T \ln \frac{2}{\delta}}.
\end{align*}
Using Lemma 3 in \cite{chu2011contextual}, we have 
\begin{align*}
    & \sum_{t=1}^{T} s_{a(t)}(t) \\
    & = \sum_{i=1}^{K} \sum_{s \in T_i(T)} s_{i}(t) \\
    & \leq \sum_{i=1}^{K} 5\sqrt{d \mid T_i(T) \mid \ln \mid T_i(T) \mid}\\
    & \leq 5\sqrt{d K T \ln T}.
\end{align*}
\begin{align*}
    & \sum_{t=1}^{T} \frac{1}{\sqrt{\mid T_{a(t)}(t) \mid}} \\
    & = \sum_{i=1}^{K} \sum_{s \in T_i(T)} s_{i}(t) \frac{1}{\sqrt{\mid T_{a(t)}(t) \mid}} \\
    & = \sum_{i=1}^{K} \sum_{s=1}^{\mid T_i(t) \mid} \frac{1}{\sqrt{s}}\\
    & \leq K \frac{1}{K}\sum_{i=1}^{K} 2 \sqrt{\mid T_i(t) \mid}\\
    & \leq 2 K \sqrt{\sum_{i=1}^{K} \frac{\mid T_i(t) \mid}{K}}\\
    & = 2 \sqrt{KT}.
\end{align*}
Hence, with probability at least $1-\frac{\delta}{2}$, we have 
\begin{align*}
    & \sum_{t=1}^{T} \Delta^{\prime}_{a(t)}(t) \\
    & = O((1+\rho+\frac{1}{\rho}) d\ln T  \ln \frac{K}{\delta}\sqrt{d K  T^{1+2\epsilon}  \ln \frac{K}{\delta} \frac{1}{\epsilon}}).
\end{align*}
Since $E^{\mu}(t)$ does not hold with probability at most $\frac{\delta}{T^2} T=\frac{\delta}{T} \leq \frac{\delta}{4}$ for $T \geq 4$, and $E^{\sigma}(t)$ does not hold with probability at most $\frac{\delta}{4}$. Therefore, both $E^{\mu}(t)$ and $E^{\sigma}(t)$ holds for all $t$ with probability at least $1-\frac{\delta}{2}$. Thus $\Delta_{a(t)}(t) = \Delta^{\prime}_{a(t)}(t)$ for all $t$ with probability at least $1-\frac{\delta}{2}$. Hence with probability at least $1-\delta$, we have
\begin{align*}
    & \sum_{t=1}^{T} \Delta_{a(t)}(t) \\
    & = O((1+\rho+\frac{1}{\rho}) d\ln T  \ln \frac{K}{\delta}\sqrt{d K  T^{1+2\epsilon} \ln \frac{K}{\delta} \frac{1}{\epsilon}}).
\end{align*}
\end{proof}

\section{Numerical Experiments}
\label{sec: numerical}
In the numerical experiment, we apply our proposed TS algorithms to a portfolio selection problem.

\subsection{Contextual MAB application to Finance}
\label{subsec: application}
Application of the bandit algorithm to the portfolio selection problem is not new. To name a few, \cite{shen2015portfolio} apply a UCB-type bandit algorithm to derive the optimal portfolio strategy that represents the combination of passive and active investments according to a risk-adjusted reward function. \cite{huo2017risk} apply a UCB-type bandit algorithm to the portfolio selection problem, under a risk-averse criterion. \cite{zhu2021online} propose an online portfolio selection method that also incorporates contextual information, based on the Exp4 algorithm presented in \cite{auer2002nonstochastic}. We adapt the portfolio selection model in \cite{huo2017risk} to our contextual setting and formally describe the problem setting below.

Consider a financial market with a large set of assets (for example, bonds, stocks and other financial derivatives), from which the portfolio manager selects to construct $K$ portfolios. Each portfolio consists of different assets with different weights. The industries are roughly divided into eleven sectors, namely energy, materials, industrials, communications, consumer discretionary, consumer staples, healthcare, financials, information technology, real estate, and utilities. At each round $t$, the manager collects information about industrial prosperity in those sectors, which makes up the contexts $x_{i}(t) \in [-1,1]^{d}$ for each portfolio $i \in [K]$, where $d \leq 11$ due to the possibility of being incapable to collect information for every sector. A larger context $x_i(t)^{j}$ indicates a better market condition for the sector $j \in [d]$. After observing the contexts, the manager chooses one portfolio to invest and receives the corresponding reward. For simplicity, we assume the reward of portfolio $i$ follows an unknown distribution $\nu_i$ with mean $x_{i}(t)^{\top}\mu_i$ and variance $\sigma_i^2$. The unknown mean parameter $\mu_i \in \mathbb{R}^{d}$ can be viewed as the sensitivity of the return to the industrial prosperity. The manager is risk-averse with a risk tolerance $\rho$. The goal is to minimize the cumulative regret over $T$ rounds under the mean-variance criterion. 

\subsection{Algorithms for Comparison}
We empirically evaluate the following algorithms in the portfolio selection problem. 
\begin{itemize}
    \item Our proposed MVTS-D algorithm (Algorithm~\ref{alg: ts_disjoint}).
    \item A variant of the TS algorithm MVTS-DN used in our regret analysis (Algorithm~\ref{alg: MVTS_DN}).
    \item TS algorithm originally designed for the risk-neutral setting. We compare with the TS algorithm from \cite{agrawal2013thompson}, referred to as TS-A. 
    \item Algorithms that make no use of the contexts. In particular, we compare with the Thompson sampling algorithm with mean-variance criterion for the context-free MAB setting (Algorithm MVTS in \cite{zhu2020thompson}). 
    \item A uniform sampling algorithm that randomly chooses an arm to pull at each round. 
\end{itemize}

To illustrate the necessity of taking into account the risk of the reward, we compare with the TS-A algorithm that works for a risk-neutral setting. At each round $t$, the TS-A algorithm samples $\widetilde{\mu}_i(t)$ from the Gaussian distribution $\mathcal{N}(\hat{\mu}_i(t), v^2 A_i(t)^{-1})$ for each arm $i \in [K]$, and plays the arm $a(t):=\arg\max x_i(t)^{\top} \widetilde{\mu}_i(t)$. Here $\hat{\mu}_i(t)=A_i(t)^{-1}b_i(t)$, the reward is assumed to follow a $R$-sub-Gaussian distribution, parameter $v=R \sqrt{\frac{24}{\epsilon} d \ln \left(\frac{1}{\delta}\right)}$, where $\delta \in (0,1)$ and $\epsilon \in (0,1)$ are two parameters used by the algorithm. To illustrate the necessity of making use of contexts that enables to learn the mean and variance parameters over time, we compare with the context-free MVTS algorithm. We include the details of the TS-A algorithm from \cite{agrawal2013thompson} and the context-free MVTS algorithm in \cite{zhu2020thompson} in Appendix~\ref{subsec: app_referenced_algorithms}. 

All the algorithms are tested on the portfolio selection problem over 100 replications. In each replication, we execute the algorithms and collect the total regrets over $T$ rounds. Parameters setting are summarized as follows: $K=10$, $d=8$, $T=10000$. All the implementing details are included in Appendix~\ref{subsec: app_experiment}. 

\subsection{Experimental Results}
\textbf{Experiment 1: evaluation of total regrets with different risk tolerances}. 
In this experiment, we evaluate the total regrets of different algorithms associated with different risk tolerances $\rho=0.1, 1, 10$. The reward distribution is Gaussian. Results are reported in Figure~\ref{fig: regret_rho}. 

\begin{figure}
    \centering
    \includegraphics[width=.45\textwidth]{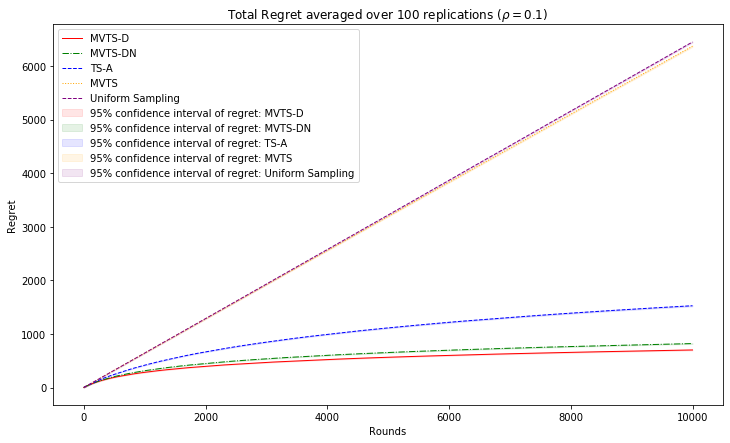}
    \includegraphics[width=.45\textwidth]{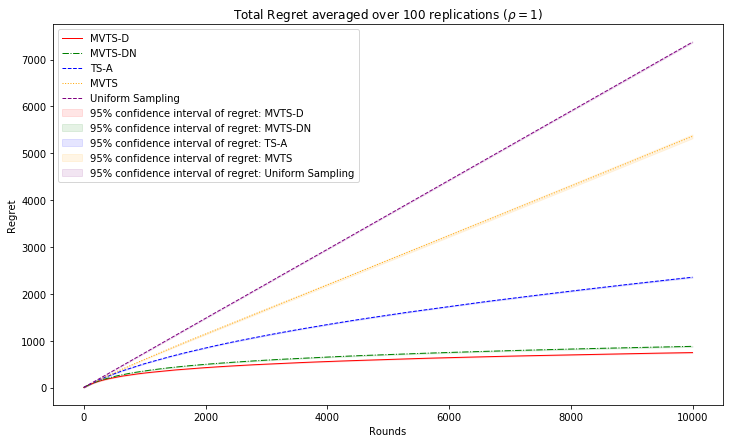}
    \includegraphics[width=.45\textwidth]{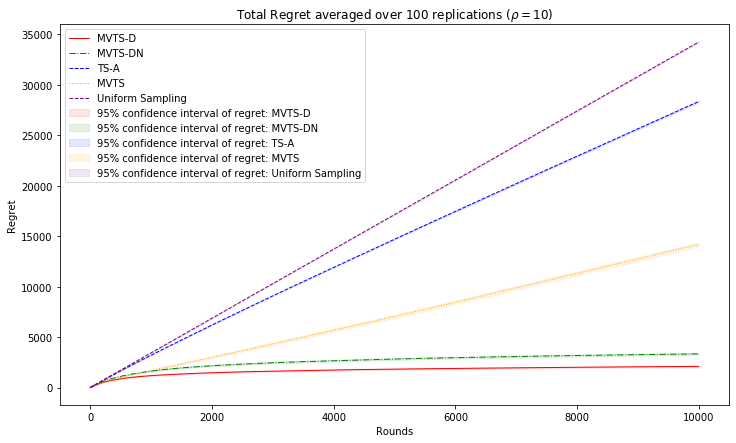}
    \caption{Total regrets comparison with different risk tolerances in the portfolio selection problem, averaged over 100 replications. }
    \label{fig: regret_rho}
\end{figure}

Figure~\ref{fig: regret_rho} shows the mean (over 100 replications) of the total regrets over time under different risk tolerances for our proposed MVTS-D and MVTS-DN algorithms, along with three benchmarks, namely TS-A, MVTS, and Uniform Sampling. From Figure~\ref{fig: regret_rho}, we have the following observations: 
\begin{itemize}
    \item Our proposed MVTS-D and MVTS-DN algorithms achieve better regrets compared to the three benchmarks in all the cases ($\rho=0.1,1,10$).
    \item As $\rho$ approaches 0 (i.e., risk-neutral case), our proposed MVTS-D and MVTS-DN algorithms behave similarly to TS-A, which corresponds to the risk-neutral case. As $\rho$ increases, our proposed MVTS-D and MVTS-DN algorithms have similarly steady performance. As for TS-A, it chooses the optimal arm according to its mean performance while overlooking the variance. As $\rho$ increases, variance tends to dominate the choice of the optimal arm. Hence, the performance of TS-A deteriorates. 
    \item In all cases ($\rho=0.1,1,10$), MVTS and Uniform Sampling behave the worst, as they make no use of the contexts and thus do not learn over time. 
\end{itemize}
\textbf{Experiment 2: evaluation of total regrets under different reward distributions}.
The reward distributions are chosen to be Gaussian, truncated normal, and uniform, respectively. Results are reported in Figure~\ref{fig: regret_distribution}. 

\begin{figure}
    \centering
    \includegraphics[width=.45\textwidth]{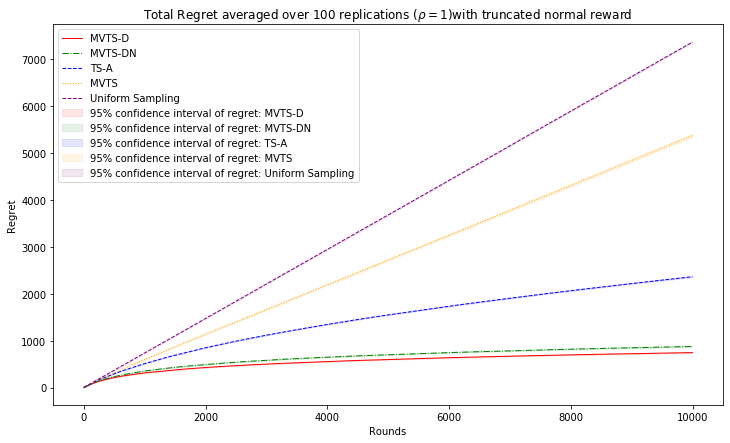}
    \includegraphics[width=.45\textwidth]{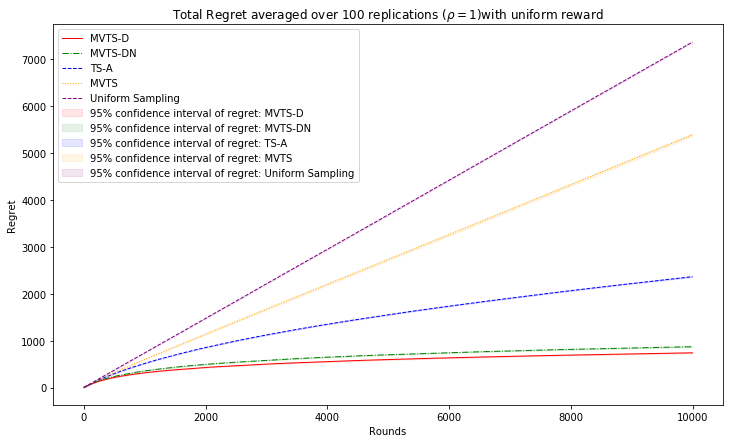}
    \caption{Total regrets comparison under different reward distributions in the portfolio selection problem, averaged over 100 replications. $\rho=1$.}
    \label{fig: regret_distribution}
\end{figure}

Figure~\ref{fig: regret_distribution} shows the mean (over 100 replications) of the total regrets over time under different reward distributions for our proposed MVTS-D and MVTS-DN algorithms, along with three benchmarks, namely TS-A, MVTS, and Uniform Sampling. From Figure~\ref{fig: regret_distribution}, we have the following observations: 
\begin{itemize}
    \item Our proposed MVTS-D and MVTS-DN algorithms are robust to model mis-specification, i.e., the assumed reward distribution is different from the true one. 
    \item Even though our regret analysis does not work for the Gaussian distribution (as the squared reward is sub-exponential instead of sub-Gaussian), in practice our proposed MVTS-DN algorithm still works well. 
\end{itemize}

\section{Conclusion}
\label{sec: conclusion}
In this paper, we apply the Thompson sampling algorithm to the contextual MAB problem under the mean-variance criterion. We show a high probability regret bound for a variant of the proposed TS algorithm. The performances of the proposed algorithm and its variant are empirically shown via a portfolio selection example, with a wide range of reward distributions. It could be an interesting future direction to apply TS algorithm, or UCB-type algorithm, to contextual MAB problem under other risk measures, such as CVaR.

\begin{appendices}
\section{Proof of Proposition~\ref{prop: posterior_disjoint}}
\label{subsec: app_disjoint}
\begin{proof}
Let $m_i(t) = A_i(t)^{-1}b_i(t)$. The prior distribution is given by:
\begin{align*}
    & \mathbb{P}(\mu_i,\lambda_i) = \mathbb{P}(\mu_i \mid \lambda_i)\mathbb{P}(\lambda_i) \\
    & = \mathcal{N}\left(m_i(t),(\lambda_i A_i(t))^{-1}\right) \cdot \text{Gamma}\left(C_i(t),D_i(t)\right) \\
    & \propto \sqrt{\lambda_i}\exp\left(-\frac{\lambda_i}{2}\left(\mu_i-m_i(t)\right)^{\top}A_i(t)\left(\mu_i-m_i(t)\right)\right) \\
    & ~~~~ \cdot \lambda_i^{C_i(t)-1}\exp\left(-D_i(t)\lambda_i\right).
\end{align*}
Similarly, the likelihood of reward $r_i(t)$ is given by:
\begin{align*}
    & \mathbb{P}(r_i(t) \mid \mu_i,\lambda_i)\\ 
    & \propto \sqrt{\lambda_i}\exp\left(-\frac{\lambda_i}{2}\left(\mu_i^{\top} x_i(t) - r_i(t)\right)^2\right).
\end{align*}
Then the posterior distribution is computed as:
\begin{align*}
    & \mathbb{P}(\mu_i,\lambda_i \mid r_i(t)) \\
    & \propto \mathbb{P}(\mu_i,\lambda_i)\cdot \mathbb{P}(r_i(t) \mid \mu_i,\lambda_i) \\
    & \propto \lambda_i^{C_i(t)}\exp\left(-\frac{\lambda_i}{2}\left(\left(\mu_i-m_i(t)\right)^{\top} A_i(t) \left(\mu_i-m_i(t)\right) \right.\right.\\
    & \left. \left.~~+ \mu_i^{\top} x_i(t)x_i(t)^{\top}\mu_i - 2\mu_i^{\top}x_i(t)r_i(t) + r_i(t)^2 + 2D_i(t)\right)\right) \\
    & = \lambda_i^{C_i(t)}\exp\left(-\frac{\lambda_i}{2}\left(\mu_i^{\top}\left(A_i(t)+x_i(t)x_i(t)^{\top}\right)\mu_i \right.\right.\\
    & \left. \left.~~- 2\mu_i^{\top}\left(b_i(t) + x_i(t)r_i(t)\right)\right.\right. \\
    & \left. \left.~~+ b_i(t)^{\top}A_i(t)^{-1}b_i(t) + 2D_i(t) + r_i(t)^2\right)\right) \\
    & = \exp\left(-\frac{\lambda_i}{2}\left(\mu_i^{\top}A_i(t+1)\mu_i - 2\mu_i^{\top} A_i(t+1)m_i(t+1) \right.\right. \\
    & \left.\left. ~~+  b_i(t)^{\top} A_i(t)^{-1}b_i(t) + 2D_i(t) + r_i(t)^2\right)\right)\lambda_i^{C_i(t)} \\ 
    & = \sqrt{\lambda_i}\exp\left(-\frac{\lambda_i}{2}\left(\left(\mu_i-m_i(t+1)\right)^{\top}A_i(t+1)\right.\right.\\
    & \left.\left.\left(\mu_i-m_i(t+1)\right) \right)\right) \cdot \lambda_i^{C_i(t+1)-1}\exp(-D_i(t+1)\lambda_i) \\
    & \propto \mathcal{N}(A_i(t+1)^{-1}b_i(t+1), (\lambda_i A_i(t+1))^{-1})\\
    & ~~ \cdot \text{Gamma}(C_i(t+1),D_i(t+1)).
\end{align*}
\end{proof}

\section{Algorithms}
\label{subsec: app_referenced_algorithms}

\begin{algorithm*}[!ht]
\SetAlgoLined
\textbf{initialization}:\\
\For{$i=1,2,\cdots,K$}
{pull arm $i$ once at round 0 and observe rewards $r_i(0)$\; 
$A_i(1)=\mathbf{I}_d + x_i(0)x_i(0)^{\top}, b_i(1)=x_i(0)r_i(0), C_i(1)=\frac{1}{2}, D_i(1)=\frac{1}{2}(r_i(0)^2-x_i(0)^{\top}A_i(1)^{-1}x_i(0)), T_i(1)=\{0\}$\;
}
\For{$t=1,2,\cdots,T$}
{
observe $K$ contexts $x_1(t),\cdots,x_K(t) \in \mathbb{R}^{d}$\;
\For{$i=1,2,\cdots,K$}
{
sample $\widetilde{\sigma}_i^2(t)$ from distribution $\mathcal{N}\left(\frac{D_i(t)}{C_i(t)}, \frac{u^2}{\mid T_i(t)\mid}\right)$\;
sample $\widetilde{\mu}_i(t)$ from distribution $\mathcal{N}\left(A_i(t)^{-1}b_i(t), v^2 A_i(t)^{-1}\right)$\;
set $\widetilde{\operatorname{MV}}_i(t) = x_i(t)^{\top}\widetilde{\mu}_i(t)-\rho \widetilde{\sigma}_i^2(t)$\;
}
play arm $a(t)=\arg\max_{i \in [K]} \widetilde{\operatorname{MV}}_i(t)$ with ties broken arbitrarily\;
observe reward $r_{a(t)}(t) \sim \nu_{a(t)}\left(x_{a(t)}(t)^{\top} \mu_{a(t)}, \sigma_{a(t)}^2\right)$\;
update ($A_i(t), b_i(t), C_i(t), D_i(t)$) according to Algorithm~\ref{alg: update_disjoint} for each arm $i$\;
set $T_{a(t)}(t+1) = T_{a(t)}(t) \bigcup \{t\}$.
}
\caption{Mean-variance Thompson Sampling for the disjoint model with variance of the reward sampled from normal distribution (MVTS-DN).}
\label{alg: MVTS_DN}
\end{algorithm*}

\begin{algorithm*}[!ht]
\SetAlgoLined
\textbf{initialization}:\\
{pull arm $i$ once at round 0 and observe rewards $r_i(0)$\; 
$A_i(1)=\mathbf{I}_d + x_i(0)x_i(0)^{\top}, b_i(1)=x_i(0)r_i(0), T_i(1)=\{0\}$\;
}
\For{$t=1,2,\cdots,T$}
{
observe $K$ contexts $x_1(t),\cdots,x_K(t) \in \mathbb{R}^{d}$\;
\For{$i=1,2,\cdots,K$}
{
compute $\hat{\mu}_i(t)=A_i(t)^{-1}b_i(t)$\;
sample $\widetilde{\mu}_i(t)$ from distribution $\mathcal{N}\left(\hat{\mu}_i(t), v^2 A_i(t)^{-1}\right)$\;
}
play arm $a(t)=\arg\max_{i \in [K]} x_i(t)^{\top}\widetilde{\mu}_i(t)$ with ties broken arbitrarily\;
observe reward $r_{a(t)}(t) \sim \nu_{a(t)}\left(x_{a(t)}(t)^{\top} \mu_{a(t)}, \sigma_{a(t)}^2\right)$\;
update $(A_i(t),b_i(t))$ according to Line~\ref{alg: posterior_update_A} and Line~\ref{alg: posterior_update_b} in Algorithm~\ref{alg: update_disjoint} for each arm $i$\;
set $T_{a(t)}(t+1)=T_{a(t)}(t) \bigcup \{t\}$.
}
\caption{TS-A algorithm from \cite{agrawal2013thompson}}
\label{alg: ts_A}
\end{algorithm*}

\label{subsec: app_MVTS}
\begin{algorithm*}[!ht]
\SetAlgoLined
\SetKwInOut{Input}{input}\SetKwInOut{Output}{output}
\Input{prior parameters ($\hat{\mu}_i(t-1),\hat{T}_i(t-1),\hat{\alpha}_i(t-1),\hat{\beta}_i(t-1)$) and new reward sample $r_{i}(t)$.}
\Output{posterior parameters ($\hat{\mu}_i(t),\hat{T}_i(t),\hat{\alpha}_i(t),\hat{\beta}_i(t)$).}
update the mean: $\hat{\mu}_{i}(t)=\frac{\hat{T}_{i}(t-1)}{\hat{T}_{i}(t-1)+1} \hat{\mu}_{i}(t-1)+\frac{1}{\hat{T}_{i}(t-1)+1} r_{i}(t)$\;
update the number of samples: $\hat{T}_i(t)=\hat{T}_i(t-1)+1$\;
update the shape parameter: $\hat{\alpha}_i(t)=\hat{\alpha}_i(t-1)+\frac{1}{2}$\;
update the rate parameter: $\hat{\beta}_i(t)=\hat{\beta}_i(t-1)+\frac{\hat{T}_{i}(t-1)}{\hat{T}_{i}(t-1)+1}\cdot\frac{(r_i(t)-\hat{\mu}_i(t-1))^2}{2}$.
\caption{Posterior updating in the MVTS algorithm.}
\label{alg: update_MVTS}
\end{algorithm*}

\begin{algorithm*}[!ht]
\SetAlgoLined
\textbf{initialization}:\\
\For{$i=1,2,\cdots,K$}
{pull arm $i$ once at round 0 and observe rewards $r_i(0)$\; 
$\hat{\mu}_i(0)=r_i(0), \hat{T}_i(0)=1, \hat{\alpha}_i(0)=\frac{1}{2}, \hat{\beta}_i(0)=\frac{1}{2}$\;
}
\For{$t=1,2,\cdots,T$}
{
\For{$i=1,2,\cdots,K$}
{
sample $\tau_{i}(t)$ from Gamma($\hat{\alpha}_i(t-1), \hat{\beta}_i(t-1)$)\;
sample $\theta_i(t)$ from $\mathcal{N}(\hat{\mu}_i(t-1), \frac{1}{\hat{T}_i(t-1)})$\;
}
play arm $a(t)=\arg\max_{i \in [K]} \theta_i(t) - \rho / \tau_{i}(t)$ and observe reward $r_{a(t)}(t)$\;
update $(\hat{\mu}_{a(t)}(t-1), \hat{T}_{a(t)}(t-1), \hat{\alpha}_{a(t)}(t-1), \hat{\beta}_{a(t)}(t-1))$ according to Algorithm~\ref{alg: update_MVTS}.
}
\caption{MVTS algorithm in \cite{zhu2020thompson}}
\label{alg: MVTS}
\end{algorithm*}

\section{Implementation Details}
\label{subsec: app_experiment}
For the portfolio selection problem, the true mean and variance parameters are summarized in Table~\ref{table: parameters}. In each replication, a new set of contexts are generated and used by all the algorithms. When the reward distribution is truncated normal, we assume it is truncated above -5 and below 5. It is worth noting that when executing the TS-A algorithm in the risk-neutral case, for a small $\epsilon$ and $\delta$, the parameter $v$ is computed so large that the total regret grows linearly. This is due to the overly large variance in the mean sampling. For meaningful experiment, we set $v=1$. For our proposed MVTS-DN algorithm, we face the same situation. Therefore, in the experiments, we set $u=1$ and $v=1$, which is different from their theoretical values.

\begin{table*}[!htp]
\centering
\begin{tabular}{llllllllll}
\toprule
             & $\mu_1$ & $\mu_2$ & $\mu_3$ & $\mu_4$ & $\mu_5$ & $\mu_6$ & $\mu_7$ & $\mu_8$ & $\sigma^2$ \\ \midrule
Portfolio 1  & 0.15  & 0.33  &-0.10  & 0.08  &-0.01 &-0.04 & 0.01 & 0.11 & 0.89 \\
Portfolio 2  & 0.14  & 0.22  & 0.15  & 0.31  & 0.02 & 0.43 &-0.08 & 0.30 & 0.66 \\
Portfolio 3  & 0.15  & 0.24  &-0.02  & 0.02  & 0.38 & 0.48 & 0.09 & 0.32 & 0.78 \\
Portfolio 4  & 0.43  & 0.44  &-0.05  &-0.08  & 0.00 & 0.43 &-0.04 & 0.15 & 0.41 \\
Portfolio 5  & 0.47  & 0.22  & 0.32  & 0.09  & 0.31 & 0.40 &-0.09 & 0.35 & 0.34 \\
Portfolio 6  & 0.49  & 0.35  & 0.07  & 0.37  &-0.04 & 0.17 & 0.45 & 0.08 & 0.91 \\
Portfolio 7  & 0.07  &-0.02  &-0.09  & 0.31  & 0.03 & 0.06 & 0.19 &-0.07 & 0.49 \\
Portfolio 8  & 0.24  &-0.01  & 0.25  & 0.32  &-0.04 & 0.15 & 0.32 & 0.15 & 0.97 \\
Portfolio 9  &-0.07  & 0.22  & 0.30  & 0.21  & 0.47 & 0.25 & 0.44 &-0.02 & 0.70 \\
Portfolio 10 &-0.02  & 0.38  & 0.14  &-0.00  & 0.46 & 0.11 & 0.35 & 0.33 & 0.66 \\ 
\bottomrule
\end{tabular}
\vspace{0.2cm}
\caption{True mean parameter $\{\mu_i\}_{i=1}^{8}$ and variance parameter $\sigma$ for all ten portfolios in the portfolio selection problem.}
\label{table: parameters}
\end{table*}

\end{appendices}

\backmatter
\bmhead{Acknowledgments}
The authors gratefully acknowledge the support by the Air Force Office of Scientific Research under Grant FA9550-19-1-0283 and Grant FA9550-22-1-0244, and National Science Foundation under Grant DMS2053489.
\bibliography{sn-bibliography} 




\end{document}